\documentclass[12pt]{article}
\usepackage{algorithm}
\usepackage[algo2e]{algorithm2e}
\usepackage[title]{appendix}
\usepackage{subcaption}
\SetKwComment{Comment}{/* }{ */}

\newcommand{\bs}{\boldsymbol{s}}
\newcommand{\bg}{\boldsymbol{g}}

\newcommand{\bm}{\boldsymbol{m}}

\newcommand{\bc}{\boldsymbol{c}}

\newcommand{\bW}{\boldsymbol{\Omega}}

\newcommand{\bX}{\boldsymbol{X}}
\newcommand{\bZ}{\boldsymbol{Z}}
\newcommand{\bx}{\boldsymbol{x}}
\newcommand{\by}{\boldsymbol{y}}
\newcommand{\bz}{\boldsymbol{z}}

\newcommand{\bphi}{\boldsymbol{\phi}}
\newcommand{\bmu}{\boldsymbol{\mu}}
\newcommand{\bsigma}{\boldsymbol{\sigma}}
\newcommand{\bzeta}{\boldsymbol{\zeta}}
\newcommand{\sumK}{\sum_{k=1}^K}
\newcommand{\prodK}{\prod_{k=1}^K}
\newcommand{\sumN}{\sum_{j=1}^{n_s}}

\newcommand{\btheta}{\boldsymbol{\theta}}
\newcommand{\bgamma}{\boldsymbol{\gamma}}

\newcommand{\RH}[1]{}

\usepackage{graphicx,psfrag,epsf}
\usepackage[round, authoryear]{natbib}
\usepackage[usenames, dvipsnames]{color}
\usepackage{url} 
\usepackage[shortlabels]{enumitem} 
\usepackage{mdframed}
\usepackage{caption}
\usepackage[labelformat=simple]{subcaption}

\usepackage{bbm}
\usepackage{siunitx}
\usepackage{float}
\usepackage[noblocks]{authblk}
\usepackage[normalem]{ulem}
\newcommand{\stkout}[1]{\ifmmode\text{\sout{\ensuremath{#1}}}\else\sout{#1}\fi}

\usepackage{cancel}

\usepackage[dvipsnames]{xcolor} %
\usepackage{diagbox}
\usepackage{bbm}

\usepackage[utf8]{inputenc}
\usepackage[english]{babel}
\usepackage{amsthm}
\usepackage{amsmath, amssymb, amsfonts, mathtools}
\usepackage{mathrsfs}

\DeclareMathOperator*{\argmin}{arg\,min}
\newtheorem{theorem}{Theorem}[section]

\newtheorem{lemma}[theorem]{Lemma}

\newtheorem{remark}{Remark}
\usepackage{hyperref}
\hypersetup{
    colorlinks=true,
    linkcolor=blue,
    citecolor=blue,
    filecolor=magenta,      
    urlcolor=blue,
}

\makeatletter
\renewcommand*\env@matrix[1][\arraystretch]{%
  \edef\arraystretch{#1}%
  \hskip -\arraycolsep
  \let\@ifnextchar\new@ifnextchar
  \array{*\c@MaxMatrixCols c}}
\makeatother

\begin{document}
\title{\bf Modeling Spatio-temporal Extremes via Conditional Variational Autoencoders}
\date{}

\author{Xiaoyu Ma, Likun Zhang, Christopher K. Wikle
}

\affil{Department of Statistics, University of Missouri, Columbia, Missouri 65211, USA}
\maketitle

\begin{abstract}
Extreme weather events are widely studied in fields such as agriculture, ecology, and meteorology. The spatio-temporal co-occurrence of extreme events can strengthen or weaken under changing climate conditions. 
In this paper, we propose a novel approach to model spatio-temporal extremes by integrating climate indices via a conditional variational autoencoder (cXVAE).  A convolutional neural network (CNN) is embedded in the decoder to convolve climatological indices with the spatial dependence within the latent space, thereby allowing the decoder to be dependent on the climate variables. 
There are three main contributions here. First, we demonstrate through extensive simulations that the proposed conditional XVAE accurately emulates spatial fields and recovers spatially and temporally varying extremal dependence with very low computational cost post training. Second, we provide a simple, scalable approach to detecting condition-driven shifts and whether the dependence structure is invariant to the conditioning variable. Third, when dependence is found to be condition-sensitive, the conditional XVAE supports counterfactual experiments allowing intervention on the climate covariate and propagating the associated change through the learned decoder to quantify differences in joint tail risk, co-occurrence ranges, and return metrics. To demonstrate the practical utility and performance of the model in real-world scenarios, we apply our method to analyze the monthly maximum Fire Weather Index (FWI) over eastern Australia from 2014 to 2024 conditioned on the El Ni\~{n}o/Southern Oscillation (ENSO) index. 

\end{abstract}

\noindent%
\vfill

\section{Introduction}\label{sec:intro}

Extreme weather events, including tornadoes, floods, thunderstorms, and heatwaves, not only cause severe environmental damage \citep{simmons2008tornado, dotzek2009overview}, but also result in loss of life, economic disruption, and displacement of communities. Therefore, it is crucial to characterize the spatial extent and temporal duration of the co-occurrence of these extreme events, which requires an accurate understanding of the extremal dependence structure over the desired region. 

To model spatial extremes, asymptotic extreme-value models such as max-stable processes \citep{davison2015statistics,davison2012statistical} or Pareto processes \citep{ferreira2014generalized, thibaud2015efficient,de2018high} have been proven to be powerful tools. These models characterize the limiting laws of multivariate/spatial extremes in the form of either renormalized
pointwise maxima or exceedances over high thresholds of spatial stochastic processes, and their theoretical properties hold in the asymptotic regime under appropriate domain-of-attraction conditions. However, this regime is often unrealistic for spatial datasets observed at finite thresholds \citep[see][for a systematic review]{huser2022advances}.

To address this limitation, sub-asymptotic models have become increasingly popular. These models are specifically designed to describe the joint tail behavior at high but finite levels. These include max-infinitely divisible (max-id) models for spatial block maxima \citep[e.g.,][]{padoan2013extreme, huser2021max, bopp2021hierarchical,zhong2022modeling}, certain types of random scale mixture models for peaks-over-threshold data \citep[e.g.,][]{Opitz2016,huser2017bridging, Huser2019}, and the spatial conditional extremes framework \citep[e.g.,][]{wadsworth2022higher,vandeskog2024efficient}. One of the key advantages of sub-asymptotic models is their ability to represent both asymptotic dependence (AD) and asymptotic independence (AI), two regimes that describe how extremes co-occur. In the AD case, the extreme events are more likely to happen jointly across locations. While in the AI case, the probability of simultaneous extremes goes to zero as the quantile threshold close to one (see formal definition in Section \ref{subsec:metric}). It's crucial for sub-asymptotic models to allow AD, AI, or both at the same time, since environmental process often exhibit AD at nearby sites and AI when locations are far apart.



Most of the existing frameworks described above still assume the underlying extremal dependence structure remains fixed over time. For example, the practical use of max-stable processes and random scale mixture models is largely restricted to purely spatial settings with independent temporal replicates. Although temporal effects are often introduced in the margins via covariates \citep[e.g.,][]{majumder2024modeling, zhang2024leveraging}, there remains a need for models that allow the process-level spatial dependence itself to evolve with changing environmental conditions. On the other hand, time series research has produced detailed specifications of extremal dependence properties for temporally indexed extremes \citep{ledford2003diagnostics,chavez2012modelling, zhang2021studying}, yet these developments have limited overlap with spatial modeling. Recent work has begun to address extremal dependence in spatio-temporal settings---for example, random scale mixture model with time-indexed radial and angular variables \citep{dell2025flexible}, hierarchical models with dependence stemming from overlap of ``slanted elliptical cylinders'' in space-time \citep{bacro2020hierarchical}, and dynamic spatio-temporal models with latent regime-switching structures \citep{yoo2025modeling}. Nevertheless, these approaches either are not necessarily realistic in how the extremal dependence evolves over time or they face severe computational challenges in high-dimensional domains. Therefore, fully flexible frameworks that allow extremal dependence parameters to vary across space and time with massive number of locations remain underdeveloped, motivating the need for new approaches.


Extending the sub-asymptotic models to large-scale spatio-temporal domains raises several major challenges, especially when we intend to allow nonstationary extremal dependence in space-time. First, high-dimensional likelihoods quickly become intractable. Indeed, when the dimension extends across both space and time, the full likelihood is generally unavailable. This also leads to the second difficulty of computational complexity. Composite likelihood methods \citep[e.g.,][]{padoan2010likelihood,castruccio2016high}, though feasible for moderate dimensions in principle, are still computationally demanding and compromise on statistical efficiency relative to the full likelihood. Additionally, sub-asymptotic models often rely on Gaussian copulas, which require costly inversion of large covariance matrices to estimate dependence parameters. Incorporating the temporal dimension substantially amplifies this computational burden. Finally, the availability of replicates in spatio-temporal datasets is extremely limited (unless we work with climate reanalysis ensembles). Typically, only a single realization is observed at each location and time, making it difficult to identify and track changes in the dependence structure. This limitation leads to high-variance estimates and can result in biased risk assessments, especially for extreme quantiles. In view of these challenges, we turn to deep learning techniques to address issues such as intractable likelihoods and excessive computational cost.

Over the past decade, deep learning has been increasingly adopted for spatial extremes and extremal-dependence modeling. These models enable researchers to better understand and predict extreme events, such as wildfires \citep{richards2024extreme, ribeiro2023reconstructing}, heavy precipitation \citep{bi2023nowcasting}, and extreme streamflow \citep{majumder2024modeling}. For \textit{spatial }extremes, \cite{boulaguiem2022modeling} applied Generative Adversarial Networks \citep[GANs;][]{goodfellow2020generative} at the copula level to learn the underlying extremal dependence, and there are no required parametric assumptions on the dependence structure. Different from the competing framework of adversarial training, Variational Autoencoders (VAEs)\citep{kingma2013auto} employ the encoder-decoder structure to reconstruct the input. \cite{zhang2023flexible} used VAE models to capture the spatial dynamics of the extremal dependence parameters within the latent space, but they do not explicitly model temporal changes in extremal dependence structure. To the best of our knowledge, there are no existing generative models that efficiently characterize the evolution of extremal dependence structures across space and time, along with the exploration of counterfactual (or storyline) experiments related to climate conditions. 

 

In this work, we develop a novel conditional VAE that integrates climate variables as conditions to model extremes and associated extremal dependence within the spatio-temporal regime, referred as conditional XVAE or cXVAE. While related to the XVAE approaches for spatial extremes \citep{zhang2023flexible}, our model is distinguished by its ability to incorporate climate drivers and capture time-varying extremal dependence. Under varying climate scenarios, the parameters that govern extremal dependence and drive extreme‑event emulations are allowed to evolve over time, thus removing the restrictive assumption of stationarity. Additionally, our model can evaluate the influence of large-scale climate conditions by comparing the reconstruction performance with and without these conditions included. This comparison provides a way to check how much the climate drivers contribute to explaining the extremes. We assess the capability of the proposed approach by generating new instances that faithfully preserve the underlying extremal‑dependence structure and estimating dependence parameters corresponding to different climate states. Computationally, the method scales to high-resolution satellite fields and other large inputs, with training and evaluation feasible on a standard laptop, thereby making a practical and accessible framework without reliance on advanced hardware.

Conducting counterfactual experiments is another benefit of our model. Such experiments directly address policy‑relevant questions. For example, if a large‑scale condition or forcing had been changed (by subtracting or adding a physically consistent perturbation) while others stayed the same, how would an extreme weather event have unfolded? These experiments are often called storyline or hindcast‑attribution experiments and have been widely employed in studies of hurricanes, compound flood‑heat events, and heatwaves, translating ``human influence'' (or any prescribed forcing)  into concrete numbers that practitioners can act on \citep[e.g.,][]{reed2020forecasted, bercos2022anthropogenic, wang2023storyline}. However, these studies are usually done through a locally calibrated climate model to pair ``factual vs. counterfactual''  simulations, which can be computationally expensive and cost hundreds of core‑hours per ensemble member. By contrast, our approach provides a powerful framework for conducting counterfactual experiments, while substantially reducing computational cost. 

The remainder of this paper proceeds as follows. Section \ref{sec:background} reviews the background of cVAE \citep{sohn2015learning} and XVAE \citep{zhang2023flexible}. Section \ref{sec:methods} details our proposed Conditional XVAE. Section \ref{sec:simulation} describes the simulation study, including emulation results and related inference. In Section \ref{sec:real data}, the proposed model will be applied to the monthly maxima of Fire Weather Index (FWI) in the eastern Australia conditioned on the El Ni\~{n}o/Southern Oscillation (ENSO) index. Finally, Section \ref{sec:discussion} concludes with a discussion of limitations and directions for future research.

\section{Background}\label{sec:background}


\subsection{Conditional VAE}\label{subsec:CVAE}

In \cite{kingma2013auto}, VAEs are designed to approximate an intractable posterior distribution and perform marginal inference through \emph{amortized}  learning. The basic structure involves encoding the incoming data into a latent distribution and then decoding the processed latent variables back to the input space to accomplish the reconstruction task.

Say we have some data $\bX = \{ \bx_t\}, t= 1, \ldots, n_t$ that are independent samples from a random vector $\bx\in \mathbb{R}^{n_s}$ and introduce per-observation latent random variables $\bz_t\in \mathbb{R}^K$. 
The VAE model introduces a recognition model (or encoder) $q_{\bphi_e}(\bz\mid\bx)$ to serve as an approximation of the true posterior $p_{\btheta}(\bz\mid\bx)$, in which $\bphi_e$ are the weights and biases in the encoder neural network and $\btheta$ consists of model parameters for the data model $p_{\btheta}(\bx\mid\bz)$ and prior model $p_{\btheta}(\bz)$. In practice, the probabilistic encoder $q_{\bphi_e}(\bz\mid\bx)$ is implemented by a multi-layer perceptron (MLP) neural network that maps $\bx$ to the parameters of a tractable variational family (e.g., a diagonal Gaussian) via the reparameterization trick:
\begin{equation}\label{eqn:encoder_form}
 \begin{split}
 \bz=\bmu + \bsigma \odot \boldsymbol{\epsilon}, \quad\boldsymbol{\epsilon}\sim \mathrm{MVN}(\boldsymbol{0}, \boldsymbol{I}),\\
    (\bmu^{\rm T},\log \bsigma^{\rm T})^{\rm T} = \mathrm{EncoderNeuralNet}_{\bphi_e}(\bx),
 \end{split}
\end{equation}
where $\odot$ is the elementwise product. After the $\bZ$ samples are drawn from the variational distribution, it is passed to the decoding structure $p_{\bphi_d}(\bx\mid\bz)$, which is referred to as a decoder, and $\bphi_d$ are the weights and biases in the decoder network (generative model). The recognition model parameters $\bphi_e$ and the generative model parameters $\bphi_d$ are both learnable.

The VAEs are trained via the optimization of the evidence lower bound (ELBO). For a single datum $\bx$, the ELBO is defined as the difference between marginal likelihoods and KL divergence of recognition model from the true posterior:
\begin{equation}\label{eqn:VAE_ELBO_1}
    \mathcal{L}_{\bphi_e, \bphi_d}(\bx) = \log p_{\bphi_d}(\bx)- D_{KL}\left( q_{\bphi_e}(\bz\mid\bx)\;\|\;p_{\bphi_d}(\bz\mid \bx)\right).
\end{equation}
The combined objective for the entire dataset 
$\sum_{i=1}^N \mathcal{L}_{\bphi_e, \bphi_d}(\bx_i)$ is then typically maximized by stochastic gradient methods with mini-batches. The maximization of the ELBO is equivalent to maximizing the marginal likelihoods while minimizing the KL discrepancy between the approximated posterior and the true posterior. To facilitate computation, the ELBO can be written equivalently as
\begin{equation}\label{eqn:VAE_ELBO_2}
     \mathcal{L}_{\bphi_e, \bphi_d}(\bx) = \mathbb{E}_{q_{\bphi_e}(\bz\mid\bx)}\left[ \log p_{\bphi_d}(\bx\mid\bz) \right] - D_{KL}\left( q_{\bphi_e}(\bz\mid\bx)\;\|\;p_{\bphi_d}(\bz)\right).
\end{equation}
Here, each expectation can be approximated using Monte Carlo:
\begin{equation}\label{eqn:VAE_ELBO_3}
    \begin{split}
        \hat{\mathcal{L}}_{\bphi_e,\bphi_d}(\boldsymbol{x})
&= \frac{1}{L} \sum_{l=1}^L \left\{\log p_{\bphi_d}(\bx\mid\bZ^{l}) -\log q_{\bphi_e}(\bZ^{l}\mid\bx)+ \log p_{\bphi_d}(\bZ^{l})\right\},
    \end{split}
\end{equation}
where $\bZ^1,\ldots,\bZ^L$ are independent draws from the encoder following~\eqref{eqn:encoder_form}. The reparameterization trick is crucial here for enabling fast computation of the gradient of $\sum_{t=1}^{n_t} \hat{\mathcal{L}}_{\bphi_e, \bphi_d}(\bx_t)$ with respect to both $\bphi_e$ and $\bphi_d$.

Extending the VAE, \cite{sohn2015learning} developed a deep conditional generative model called a Conditional Variational Autoencoder (cVAE). This model is capable of learning the conditional distribution $p_{\bphi_d}(\bx\mid\boldsymbol{c})$, allowing for the generation of samples with respect to certain conditions $\boldsymbol{c}$. The setup of encoders and decoders of a cVAE are inherited from the VAE framework. The ELBO is adjusted to be conditioned on $\bc$:
\begin{equation}\label{eqn:CVAE_ELBO_1}
\begin{split}
     \mathcal{L}_{\bphi_e, \bphi_d}(\bx\mid \bc) &= \log p_{\bphi_d}(\bx\mid \bc)- D_{KL}\left( q_{\bphi_e}(\bz\mid\bx,\bc)\;\|\;p_{\bphi_d}(\bz\mid \bx,\bc)\right)\\
     &=\mathbb{E}_{q_{\bphi_e}(\bz\mid\bx,\bc)}\left[ \log p_{\bphi_d}(\bx\mid\bz,\bc) \right] - D_{KL}\left( q_{\bphi_e}(\bz\mid\bx,\bc)\;\|\;p_{\bphi_d}(\bz\mid\bc)\right),
\end{split}
\end{equation}
and a Monte Carlo estimator analogous to \eqref{eqn:VAE_ELBO_3} also applies with draws $\bZ^{l}\sim q_{\bphi_e}(\bz\mid\bx,\bc)$ and (optionally) a condition-dependent prior $p_{\bphi_d}(\bz\mid\bc)$.

When the data distribution changes with observed drivers $\bc$ (e.g., circulation indices, SSTs, season), the VAE without conditioning marginalizes out condition-specific structure and prevents controlled generation under specified conditions. A cVAE remedies this by conditioning both encoder and decoder on $\bc$ to learn $p_{\btheta}(\bx\mid\bc)$, which enables \emph{scenario-controlled} simulation and counterfactuals by intervening on $\bc$, and allows dependence parameters to vary with $\bc$, accommodating nonstationarity. Practically, the conditioning also reduces posterior variance by stratifying the latent representation with informative covariates, while still sharing strength across nearby $\bc$.
There are variants of cVAE model that are designed to adopt various objectives such as segmentation recognition \citep{sohn2015learning}, predicting the remaining useful life of complex systems \citep{wei2021learning} and next-state emulation in physics-based character controllers \citep{won2022physics}. 

\subsection{XVAE}\label{subsec:XVAE}

\cite{zhang2023flexible} propose the extremes VAE (XVAE) to model high-dimensional spatial extremes through a hybrid architecture that embeds a max-infinitely divisible (max-id) model within a VAE. In the autoencoder architecture, the low‑rank latent representation learned by the encoder is mapped through a decoder that embeds a max‑id construction, which allows the model learn the extremal dependence structure outperforming Gaussian and max-stable processes or standard deep generative models.

In XVAE, the spatial observation (max-id) model is defined as
\begin{equation}\label{eqn:XVAE_X(s)}
    X(\bs)=\epsilon(\bs)Y(\bs),\;\bs\in\mathcal{S},
\end{equation}
where $\mathcal{S}\in \mathbb{R}^2$ is the domain of interest and $\epsilon(\bs)$ is a noise process with independent Fr\'{e}chet$(0,\tau,\alpha_0)$ marginal distributions:
\begin{equation}\label{eqn:frechet}
    \mathbb{P}\{\epsilon(\bs)\leq x\}=\exp\{-(x/\tau)^{-\alpha_0}\},
\end{equation}
where $x>0$, $\tau>0$ and $\alpha_0>0$. The process $Y(\bs)$ is constructed with a low-rank representation:
\begin{equation}\label{eqn:XVAE_Y(s)}
    Y(\bs)=\left\{\sum_{k=1}^K \omega_k(\bs)^{\frac{1}{\alpha}}Z_{k}\right\}^{\alpha_0},
\end{equation}
where $\alpha\in (0,1)$, and $\{\omega_k(\bs): k=1,\ldots,K\}$ are fixed compactly-supported radial basis functions (RBFs) centered at $K$ pre-specified knots. The latent variables are defined as exponentially-tilted positive-stable variables \citep{hougaard1986survival}, denoted as $Z_{k}\stackrel{\text{ind}}{\sim} \mathrm{expPS}(\alpha,\theta_k),\; k=1,\ldots, K$. The parameter $\alpha$ determines the tail behavior such that smaller $\alpha$ will lead to heavier tail. The tilting parameters $\theta_k\geq 0$ determine the extent of tilting, with larger values of $\theta_k$ leading to lighter-tailed $Z_k$. Additionally, the density of $Z_k$ is of the form
\begin{equation}\label{eqn:XVAE_h}
    h(z; \alpha, \theta_k) = \frac{f_{\alpha}(z) \exp (-\theta_k z)}{\exp (-\theta_k^{\alpha})}, \quad z>0, \quad k=1,\ldots,K,
\end{equation}
and $f_{\alpha}$ is the density function of a positive-stable variable equipping with its Laplace transform $\int_{\mathbb{R}}\exp(-sx)f_{\alpha}(x)\mathrm{d}x = \exp(-s^{\alpha}), s\geq 0$.

To accommodate the extremes framework, the encoding-decoding VAE structure is modified. For $t=1,\ldots,n_t$, the encoder is defined
\begin{equation}\label{eqn:XVAE_encoder}
 \begin{split}
 \bz_t&=\bmu_t + \bzeta_t \odot \boldsymbol{\eta}_t, \\
    \eta_{kt}&\stackrel{\text{i.i.d.}}{\sim} \mathrm{Normal}(0,1),\\
    (\bmu_t^\top,\log \bzeta_t^\top)^\top &= \mathrm{EncoderNeuralNet}_{\bphi_e}(\bx_t),
 \end{split}
\end{equation}
where $\odot$ is the elementwise product, and the encoder neural network is constructed with a fully-connected MLP network.

Unlike the Gaussianity assumption of a vanilla VAE, the latent variables are assigned to exponentially-tilted positive-stable distributions. Therefore the prior model for latent process is
\begin{equation}\label{eqn:XVAE_prior}
    p_{\bphi_d}(\bz_t) = \prodK h(z_{kt};\alpha_t, \gamma_{kt}),
\end{equation}
where $h(\cdot;\alpha_t, \gamma_{kt})$ is the density function of $\mathrm{expPS}(\alpha_t,\gamma_{kt})$.

The decoder is based on the spatial observation model we defined in Equation \eqref{eqn:XVAE_X(s)}:
\begin{equation}\label{eqn:XVAE_decoder}
     p_{\bphi_d}(\bx_t\mid\bz_t) = \left(\frac{1}{\alpha_0}\right)^{n_s} \left\{\prod_{j=1}^{n_s}  \frac{1}{x_{jt}}\left(\frac{x_{jt}}{\tau y_{jt}}\right)^{-1/\alpha_0}\right\} \exp \left\{ -\sumN \left(\frac{x_{jt}}{\tau y_{jt}}\right)^{-1/\alpha_0}\right\},
\end{equation}
where $y_{jt}=\sumK\omega_{kj}^{1/{\alpha}_t}z_{kt}$. The dependence parameter estimations and the reconstruction of the inputed are achieved via two separate neural networks
\begin{equation}\label{eqn:decoder_form}
    \begin{split}
        (\hat{\alpha}_t, \hat{\bgamma}_t^\top)^\top &= \mathrm{DecoderNeuralNet}_{\bphi_{d,0}}(\bZ_t),\\
        \hat{\bX}_t &= \mathrm{DecoderNeuralNet}_{\bphi_{d,1}}(\bZ_t),
    \end{split}
\end{equation}
where $\bphi_{d}=(\bphi_{d,0}^\top, \bphi_{d,1}^\top)^\top$ are the bias and weight parameters of the decoder neural networks.

The ELBO loss function can be calculated with the independence draws of $\bZ^1, \ldots, \bZ^L$ according to Equation (\ref{eqn:VAE_ELBO_3}). The parameters of the encoder and decoder networks, $\bphi_e, \bphi_d$ are updated via stochastic gradient descent algorithm. Uncertainty quantification is obtained from the repeated estimates of dependence parameters $\alpha_t, \gamma_t$ from the samples of $\bZ_t$ via Equations \eqref{eqn:XVAE_encoder} and \eqref{eqn:decoder_form}.

\section{Methodology}\label{sec:methods}

\subsection{Log-Laplace noise process}

The spatial extremes process embedded in our novel conditional XVAE retains the flexible max-id backbone of \citet{zhang2023flexible}, with one key modification: we replace the noise process $\epsilon(\bs)$ with independent log-Laplace$(0, 1/\alpha_0)$ marginal distributions
\begin{align}\label{eqn:logLaplace}
    \mathbb{P}\left(\epsilon(\bs)\leq x\right)=     \begin{cases}
    \frac{1}{2}\exp (\alpha_0 \log x), \quad &0 < x \leq 1, \\
    1-\frac{1}{2}\exp(-\alpha_0 \log x), \quad &x > 1,
    \end{cases}
\end{align}
where $x>0$ and $\alpha_0 > 0$. The choice of this noise process is preferred to the Fr\'{e}chet distribution because the density of log-Laplace distribution is symmetric around 1, which naturally represents balanced deviations without inflation or deflation. In contrast, the Fr\'{e}chet distribution is highly right-skewed with concentrating mass near zero and thus induces unbalanced multiplicative perturbations. Moreover, the log-Laplace distribution offers direct control over tail behavior of the noise process through its scale parameter, $1/\alpha_0$. The flexible adjustment between heavier or lighter tails is dominated by this scale parameter, e.g. smaller scale makes the error more concentrated around 1 (lighter tail).

Crucially, we show in Appendix~\ref{App:log-Laplace} that Fr\'{e}chet$(0,\tau,\alpha_0)$ and log-Laplace$(0, 1/\alpha_0)$ have the same tail index $\alpha_0$. Furthermore, the following result implies that replacing Fr\'echet noise with log-Laplace noise of the same tail index $\alpha_0$ leaves the flexible tail behavior unchanged, both marginally and jointly.

\begin{theorem}[Tail equivalence under noise replacement]\label{thm:tail_equiv}
Let $\{Y(\bs):\bs\in\mathcal{S}\}$ be a nonnegative random field that satisfies for each $\bs$, $\mathbb{E}\{Y(\bs)^{\alpha_0+\eta}\}<\infty$ for some $\eta>0$, and for all pairs $(\bs_1,\bs_2)$,
$\mathbb{E}\{Y(\bs_1)^{\alpha_0}Y(\bs_2)^{\alpha_0}\}<\infty$. Let
\begin{equation*}
X_F(\bs)=\epsilon_F(\bs)\,Y(\bs),\qquad X_L(\bs)=\epsilon_L(\bs)\,Y(\bs),
\end{equation*}
where $\{\epsilon_F(\bs)\}$ and $\{\epsilon_L(\bs)\}$ are i.i.d. across $\bs$, independent of\; $Y$, and have regularly varying tails with the same index $\alpha_0>0$:
\begin{equation}\label{eqn:tail_cond}
\bar F_{\epsilon_F}(x)\sim c_F\,x^{-\alpha_0},\qquad 
\bar F_{\epsilon_L}(x)\sim c_L\,x^{-\alpha_0}\quad(x\to\infty),
\end{equation}
for some $c_F,c_L\in(0,\infty)$. 
Then, as $x\to\infty$,
\begin{equation*}
\bar F_{X_F(\bs)}(x)\ \sim\ \frac{c_F}{c_L}\,\bar F_{X_L(\bs)}(x).
\end{equation*}

Moreover, for any $\bs_1,\bs_2\in\mathcal{S}$,
\begin{equation*}
\mathbb{P}\{X_F(\bs_1)>x,\ X_F(\bs_2)>x\}\ \sim\ \Big(\frac{c_F}{c_L}\Big)^{\!2}
\ \mathbb{P}\{X_L(\bs_1)>x,\ X_L(\bs_2)>x\}.
\end{equation*}
\end{theorem}

\begin{remark}
The proof of this Theorem is detailed in Appendix~\ref{sec:proofs}. For $\epsilon_F\sim\text{Fr\'echet}(0,\tau,1/\alpha_0)$,
$\bar F_{\epsilon_F}(x)\sim \tau^{\alpha_0}x^{-\alpha_0}$ so $c_F=\tau^{\alpha_0}$.
For $\epsilon_L\sim\text{log-Laplace}(0,1/\alpha_0)$,
$\bar F_{\epsilon_L}(x)=\tfrac12 x^{-\alpha_0}$ for $x>1$, so $c_L=\tfrac12$. Therefore, replacing Fr\'echet noise by log-Laplace noise with the same tail index $\alpha_0$ preserves the marginal tail decay rate $x^{-\alpha_0}$ and the bivariate joint tail decay rate $x^{-2\alpha_0}$; only the multiplicative tail constants differ.
\end{remark}


\subsection{Conditional XVAE with learnable basis functions}\label{subsec:Conditional XVAE}

Here, we focus on reconstructing the observed data and the extremal dependence structures as well as on understanding the influence of climate variables.
Denote the data realization as $X(\bs)$ and the conditions as $\boldsymbol{c}$. The data-level model is formulated as
\begin{equation}\label{eqn:model_X}
X(\bs)\mid\boldsymbol{c}=\epsilon(\bs)\times \left(Y(\bs)\mid\boldsymbol{c}\right),\;\bs\in\mathcal{S},
\end{equation}
where $\mathcal{S}\in \mathbb{R}^2$ is the desired spatial domain, and $\{\epsilon(\bs)\}$ is the independent log-Laplace process defined in Equation~\eqref{eqn:logLaplace}. Note that we are using the notation $(Y(\bs)|\bc)$ in (\ref{eqn:model_X}) to denote the dependence of $Y$ on the conditioning variables, $\bc$. Then, the process $\{Y(\bs)\}$ conditioned on $\bc$ is constructed via a low-rank representation:
\begin{equation}\label{eqn:model_Y}
    Y(\bs)\mid \bc=\sum_{k=1}^K \omega_k(\bs)(Z_k \mid\bc), \quad Z_k \mid\boldsymbol{c}\stackrel{\text{ind}}{\sim} \text{expPS}(\alpha, \theta_k(\boldsymbol{c})),
\end{equation}
where $k=1,\ldots, K$, $0 < \alpha < 1$, and $\delta>0$ and $\theta_k(\boldsymbol{c}) \geq 0$ changes with the condition $\bc$. The set $\{\omega_k(\bs):\bs\in\mathcal{S},\ k=1,\ldots,K\}$ are basis spatial basis functions (i.e., RBFs in this application) that map the observations from the physical space to the latent space, with $Z_k$ the basis projection coefficients (again, we explicitly denote the dependence of these coefficients on the condition variables). The choice between local and global basis functions depends on the modeling objectives and the needs of interpretability. In \citet{zhang2023flexible}, the basis functions are predefined as compactly supported RBFs, which are efficient but not flexible enough. In this work, instead of fixing the basis functions, we treat $\boldsymbol{W}=\{\omega_k(\bs_j): j=1,\ldots,n_s,\ k=1,\ldots,K\}$ as \textit{unknown} and \textit{learnable} parameters within the conditional XVAE framework, optimized jointly with the network weights (see Figure~\ref{fig:diagram}). This formulation represents a potentially significant advancement over previous work, as it allows the model to learn, in an adaptive manner, how the latent space interacts with the observed process for different datasets, rather than imposing a fixed spatial structure. This improvement also preserves the interpretability of $\boldsymbol{W}$. 

The latent variables $Z_k$ are still assumed to have exponentially-tilted positive-stable distributions, but the conditions $\boldsymbol{c}$ are allowed to impact the tilting parameter $\theta_k$.
Consequently, the tail thickness of the latent factors $\{Z_k:k=1,\ldots, K\}$ becomes \emph{explicitly condition-indexed} and hence the spatial clustering and strength/range of extremal co-occurrence changes over time. In particular, the pairwise tail summaries (e.g., $\chi_{ij}(\bc)$ and $\eta_{ij}(\bc)$) are allowed to vary with $\bc$, permitting condition-dependent behavior and even AD/AI transitions as climate drivers change.


As described in Section~\ref{subsec:CVAE}, a cVAE reconstructs inputs while learning a conditionally structured latent space. For our model~\eqref{eqn:model_X}-\eqref{eqn:model_Y}, the objects of primary interest are the condition-dependent latent variables and tilting parameters that control extremal dependence. Let's consider the data $\bx_t = \{ X_t(\bs_j) : j=1, \ldots, n_s \}$ at time $t=1,\ldots,n_t$, and latent variables $\bz_t = \{ Z_{kt} : k = 1, \ldots, K \}$. The condition $\boldsymbol{c}_t$ is in general a multivariate time series, e.g., an ENSO climate index or time-varying spatial fields.

\paragraph{Approximate Posterior/Encoder ($q_{\bphi_e}(\bz_t\mid \bx_t,\bc_t)$):} The encoder maps the observed field and conditions $(\bx_t,\bc_t)$ to a variational posterior over the latent vector $\bz_t\in \mathbb{R}^K$:  
\begin{equation}\label{eqn:encoer_form_xvae}
\begin{split}
\log \bz_t=\log \bmu_t +g(\bc_t)+ \bsigma_t \odot \boldsymbol{\epsilon}_t, \quad\boldsymbol{\epsilon}_t\stackrel{\text{ind}}{\sim} \mathrm{MVN}(\boldsymbol{0}, \boldsymbol{I}),\\
    (\bmu_t^{\rm T},\bsigma_t^{\rm T})^{\rm T} = \mathrm{EncoderNeuralNet}_{\bphi_e}(\bx_t),
\end{split}
\end{equation}
where $t=1,\ldots,n_t$. The vectors $\bmu_t\in \mathbb{R}^K$ and $\log \bsigma_t\in \mathbb{R}^K$ are obtained from the input through the encoder neural network, for which we deploy a fully connected MLP with Softplus activations. We choose a MLP instead of a convolutional neural network (CNN) here so the encoder can naturally handle inputs defined on irregular point sets rather than on a regular grid. The Softplus activation ensures the realizations of $\bmu_t$ and $\bsigma_t$ are positive. Compared to~\eqref{eqn:encoder_form}, we employ the reparameterization trick on the log scale to ensure the non-negativity of $\bz_t$ given the auxiliary variable $\boldsymbol{\epsilon}_{t}$ is sampled from the standard normal distribution. 
Additionally, the term $g(\cdot)$ in~\eqref{eqn:encoer_form_xvae} introduces the condition $\bc_t$ into the $K$-dimensional latent space through a linear mapping. In our formulation, $g(\cdot)$ provides a simple and direct mechanism for injecting covariate information into the latent representation. As noted in~\ref{eqn:CVAE_ELBO_1}, many alternative designs could be used to incorporate conditions into the latent space—for example, nonlinear transformations or more expressive networks. Our choice of a linear mapping is a problem-specific, computationally efficient design tailored to the needs of this project. See the top left panel of~Figure~\ref{fig:diagram} to see schematic flow of the encoding process.
\begin{figure}[!t]
    \centering
    \includegraphics[width=0.99\linewidth]{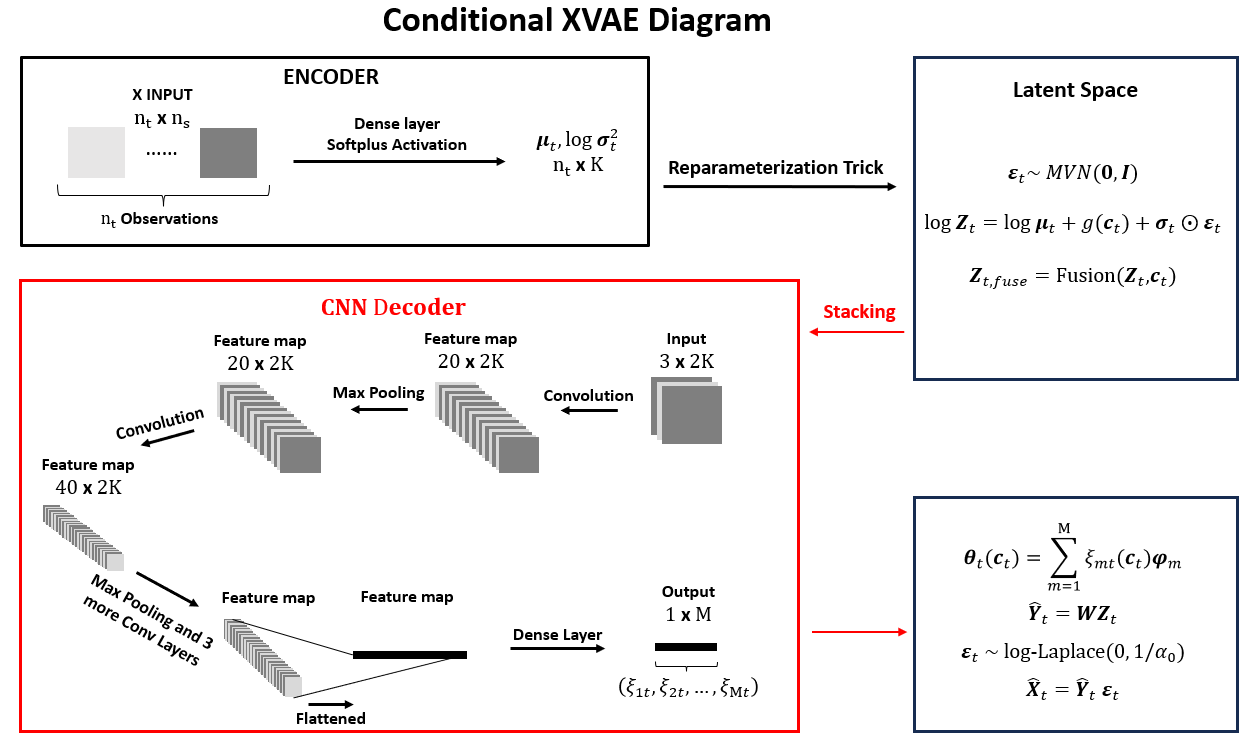}
    \caption{Conditional XVAE architecture with three main components: an encoder, a latent space, and a CNN decoder. \textbf{Encoder} (top left): For each time $t=1,\ldots,n_t$, the input spatial fields $\bX_t$ of $n_s$ locations is mapped through a dense layer with Softplus activation to produce mean $\bmu_t$ and log-variance $\log \bsigma_t$ vectors of dimension $K$. \textbf{Latent Space} (top right): Latent variables are constructed on the log scale with transformed conditions $g(\bc_t)$ and fused with $\bc_t$ as in~\eqref{eqn:cXVAE_fuse}. \textbf{CNN decoder} (red box): The fused latent variables are stacked and transposed to form structured inputs. Convolution and max pooling layers extract feature maps (e.g., from $3 \times 2K$ to $40 \times 2K$), which are flattened and passed through a dense layer to yield coefficients $\{ \xi_{1t}, \xi_{2t}, \ldots, \xi_{Mt}\}$, as defined in~\eqref{eqn:CNNdecoder}. The generative process is summarized in the bottom right. The tilting parameters $\btheta_t$ that control extremal dependence are estimated via pre-specified basis functions $\varphi_{mt}$. The de-noised response $\boldsymbol{Y}_t$ is obtained as a linear combination of latent variables $\bZ_t$ and learnable weights $\boldsymbol{W}$, and the response surface $\bX_t$ is generated by introducing log-Laplace noise $\boldsymbol{\epsilon}_t$.}
    \label{fig:diagram}
\end{figure}

\paragraph{Data Model/CNN Decoder ($p_{\bphi_d}(\bx_t\mid \bz_t,\bc_t)$):} The choice of CNNs here is crucial for several reasons. First, convolutions preserve local spatial structure, which is essential for representing the clustering behavior of extremes. Second, CNNs require far fewer parameters (weights, biases, kernels) than fully connected architectures. By learning local kernels shared across the domain, they can extract complex spatial features without treating each latent variable independently. This parameter efficiency not only reduces memory requirements but also improves stability during training, helping to mitigate issues such as exploding or vanishing gradients. Finally, CNNs are particularly effective at detecting rare but strong signals, a defining characteristic of extremes. Convolutional layers are naturally dominated by regions where extreme events occur, amplifying their influence on subsequent layers. In contrast, fully-connected layers average signals across all locations, eliminating the effect of extremes. From this perspective, CNNs provide a principled mechanism for enhancing sensitivity to extremes and improving the quality of the estimated random coefficients.

To strengthen the effect of the conditions $\bc_t$, we fuse the encoded latent variable $\bz_t$ with $\bc_t$ after the reparameterization trick. Without loss of generality, we consider the case of scalar climate conditions, $c_t \in \mathbb{R}$. The fusion interleaves latent and condition variables
\begin{equation}\label{eqn:cXVAE_fuse}
    \bz_{t,\text{fuse}}(c_t) = (z_{1t}, c_t, z_{2t}, c_t, \ldots, z_{Kt}, c_t)^\top,
\end{equation}
which is of dimension $2K \times 1$. Interleaving provides an effective way to inject conditioning information while preserving the encoded features. As the fused latent variables pass through the convolutional layers, the conditioning variables strongly influence the output because they contribute directly to the combinations alongside the latent variables when passing through the kernel filters during the convolution process.

The tilting parameter field $\btheta_t\in \mathbb{R}^K$ is too high-dimensional to learn directly, so we represent it with lower-dimensional basis function representation:
\begin{align}\label{eqn:theta_expansion}
    {\btheta}_t(\bc_t) = \sum_{m=1}^M \xi_{mt}(\bc_t) \boldsymbol{\varphi}_{m},
\end{align}
in which $M\leq K$, $\{\boldsymbol{\varphi}_{m}\in\mathbb{R}^K:m=1,\ldots, M\}$ are radial basis functions discretized over the latent space, and the coefficients $\boldsymbol{\xi}_t(\bc_t)=(\xi_{1t}(\bc_t),\ldots, \xi_{Mt}(\bc_t))^\top$ are the output of one of the decoder neural networks
\begin{equation}\label{eqn:CNNdecoder}
    \boldsymbol{\xi}_t(\bc_t) = \text{DecoderNeuralNet}_{\bphi_{d,0}}(\bz_{t-1,\text{fuse}}(\bc_{t-1}), \bz_{t,\text{fuse}}(\bc_t), \bz_{t+1,\text{fuse}}(\bc_{t+1})).
\end{equation}
To enforce richer temporal context, the decoder neural network in~\eqref{eqn:CNNdecoder} concatenates fused latent fields from three consecutive time steps and uses them as short-window pseudo-replicates. This assumes the latent process evolves smoothly and is more informative than a single snapshot. 

Lastly, the other decoder neural network for reconstructing the process $Y_t(\bs)$ is a linear mapping:
\begin{equation}
    \by_{t}(\bc_t) = \boldsymbol{W}\bz_t,
\end{equation}
where the architecture is consistent with~\eqref{eqn:model_Y} and $\boldsymbol{W}$ are unknown weights. In this case, $\bW$ corresponds to $\bphi_{d,1}$ in Equation~\eqref{eqn:decoder_form}. Then we impose the flexible extreme model introduced in~\eqref{eqn:model_X}-\eqref{eqn:model_Y} on the decoder
\begin{equation}\label{eqn:decoder_CDF}
p(\bX_t \leq \bx_t \mid \bz_t,\bc_t) = \prod_{j\in \mathcal{J}_t} \left(\frac{1}{2} x_{jt}^{\alpha_0}  y_{jt} ^{-1} \right) \cdot  \prod_{j\not\in \mathcal{J}_t} \left(1-\frac{1}{2} x_{jt}^{-\alpha_0} y_{jt} \right),
\end{equation}
where the conditional CDF is derived using the CDF of log-Laplace distribution, $y_{jt}$ is the corresponding element in the process $\by_{t}(\bc_t)$ and $\mathcal{J}_t = \{ j \in (1,\ldots,n_s): 0<x_{jt}/y_{jt}<1\}$. Differentiating (\ref{eqn:decoder_CDF}) over $\bx_t$ gives
\begin{equation}\label{eqn:decoder_PDF}
    p_{\bphi_d}(\bx_t \mid \bz_t,\bc_t) = \prod_{j\in \mathcal{J}_t} \frac{\alpha_0 x_{jt}^{\alpha_0-1}}{2y_{jt}^{\alpha_0}} \cdot \prod_{j\not\in \mathcal{J}_t} \frac{\alpha_0 x_{jt}^{-\alpha_0-1}}{2y_{jt}^{-\alpha_0}}.
\end{equation}

\paragraph{Prior on Latent Process ($p_{\bphi_d}(\bz_t\mid \bc_t)$):} As mentioned in Equation~\eqref{eqn:model_Y}, the prior distribution $p_{\bphi_d}(\bz_t \mid \boldsymbol{c}_t)$ is the exponentially-tilted positive-stable distribution, and we denote the density of it as $h(z_{kt}; \alpha,\theta_{kt}(\bc_t))$. The joint density of $\bz_t$ is
\begin{equation}\label{eqn:prior}
        p_{\bphi_d}(\bz_t \mid \boldsymbol{c}_t) = \prodK h(z_{kt};\alpha,\theta_{kt}(\bc_t)),
\end{equation}
for $t = 1,\ldots,n_t$. 

As in~\eqref{eqn:CVAE_ELBO_1}, combining the forms of $q_{\bphi_e}(\bz_t\mid \bx_t,\bc_t)$ in~\eqref{eqn:encoer_form_xvae}, $p(\bx_t \mid \bz_t,\bc_t)$ in~\eqref{eqn:decoder_PDF} and $p_{\bphi_d}(\bz_t \mid \boldsymbol{c}_t)$ in~\eqref{eqn:prior} yields a Monte-Carlo ELBO that enforces the model’s structure (basis representation, decoder likelihood, and parameterization). 

Beyond the standard ELBO, we introduce an additional penalty term to enforce temporal continuity in the estimated dependence parameters. The rationale is that the underlying extremal dependence structure is expected to evolve smoothly over time. The irregular fluctuations in the coefficient estimates may reflect noise or instability in the optimization. To encourage coherence in time, we penalize discrepancies between adjacent time points in the coefficient vector $\boldsymbol{\xi}_t(\bc_t)=(\xi_{1t}(\bc_t),\ldots, \xi_{Mt}(\bc_t))^\top$. Specifically, we define the penalty as
\begin{equation*}
\rho_t = \rho_0 \sum_{m=1}^M \frac{\xi_{mt}(\bc_t) - \xi_{m(t-1)}(\bc_t)}{c_t - c_{t-1}},
\end{equation*}
which measures the average normalized difference of coefficients across consecutive times. The influence of this penalty is controlled by a hyperparameter $\rho_0$, which balances the contribution of the standard ELBO (fitting the data) against the temporal smoothness of the parameter estimates. This continuity term is then subtracted from the ELBO objective (since the ELBO is maximized), providing a form of temporal regularization that discourages irregular jumps while preserving flexibility for gradual changes of $\boldsymbol{\xi}_t(\bc_t)$ (and consequently in $\btheta_t(\bc_t)$):
\begin{equation}\label{eqn:cXVAE_ELBO_1}
    \mathcal{L}_{\bphi_e, \bphi_d}^{\star}(\bx_t \mid \bc_t) = \mathcal{L}_{\bphi_e, \bphi_d}(\bx_t \mid \bc_t) - \rho_t,
\end{equation}
where $\mathcal{L}_{\bphi_e, \bphi_d}(\bx_t \mid \bc_t)$ is formed as in Equation~\eqref{eqn:CVAE_ELBO_1}. Full derivations of $\mathcal{L}_{\bphi_e, \bphi_d}^{\star}(\bx_t \mid \bc_t)$ are provided in Appendix~\ref{App:ELBO}.

We implement the Conditional XVAE algorithm in PyTorch \citep{paszke2019pytorch} that utilizes tape-based autograd (reverse mode automatic differentiation). With the optimization of ELBO described in Equation~\eqref{eqn:cXVAE_ELBO_1}, the weight and bias parameters defined in the encoder and decoder neural networks are updated with the stochastic gradient descent (SGD) algorithm. During the update of parameters, the Adam optimizer \citep{kingma2014adam} was used to adjust the learning rates and the proportion of updates.

\subsection{Evaluation Metrics}\label{subsec:metric}

To highlight the performance of our model and demonstrate the contribution of climate conditions, we compare the results obtained when incorporating the observed climate conditions with those obtained using \textit{white noise conditions} of the same scale. This comparison shows how much the inclusion of climate conditions contributes to the overall performance of the model.

To measure the performance of our model, we first examine the alignment of the extreme event occurrences of the truth and emulation by evaluating $\chi_{ij}(u)$. The coefficient $\chi_{ij}(u)$ identifies the extremal dependence between two random variables $X_i$ and $X_j$ (which may correspond to the observations at two different spatial locations):
\begin{equation*}
    \chi_{ij}(u) = \mathbb{P}\left\{ F_j(X_j) > u \mid F_i(X_i) > u \right\},
\end{equation*}
for some threshold $u \in (0, 1)$ and $F_j, F_i$ are the marginal distribution functions for variables $X_j$ and $X_i$, respectively. The coefficient $\chi$ represents the conditional probability that an extreme event at location $j$ occurs given an extreme event has occurred at location $i$. As $u \xrightarrow{}1$, $X_i$ and $X_j$ are said to be asymptotically independent (AI) if $\chi_{ij} = 0$, and asymptotic dependent (AD) if $\chi_{ij} > 0$.

Then, to characterize the overall dependence strength within the spatial domain, we use the metric of the averaged radius of exceedances (ARE) proposed by \cite{zhang2022accounting}. This metric measures the joint exceedance of quantile $u$ of the empirical cumulative distribution function (CDF) with respect to an arbitrary reference point. Say we have a number of independent replicates for every cell on a regular grid $\mathcal{G} = \{g_i \in \mathcal{S} : i = 1, \ldots, n_g\}$ over the domain $\mathcal{S}$ with side length of the grid unit $\psi > 0$, the number of replicates is denoted by $n_r$, over the total $n_g$ grid cells. Then the vector of realizations at cell $g_i$ is denoted by $\mathbf{X}_r = \{X_r(g_i) : i = 1, \ldots, n_g\}$, $r = 1, \ldots, n_r$. The empirical marginal distribution function of the cell $g_i$ can be calculated via
\begin{equation*}
    \hat{F}_i(x) = n_r^{-1} \sum_{r=1}^{n_r} \mathbbm{1}(X_r(g_i) \le x),
\end{equation*}
where $\mathbbm{1}\{\cdot\}$ is the indicator function. For each cell $g_i$, we then transform $(X_1(g_i), \ldots, X_{n_r}(g_i))^\top$ to the uniform scale via $U_{ir} = \hat{F}_i(X_r(g_i))$, $r = 1, \ldots, n_r$. Let $\mathbf{U}_r = \{U_{ir} : i = 1, \ldots, n_g\}$ and $U_{0r} = \hat{F}_0\{X_r(\bs_0)\}$, which is the empirical marginal distribution function at the arbitrary reference point $\bs_0$. The ARE metric at the threshold $u$ is defined by
\begin{equation}
\widehat{\mathrm{ARE}}_{\psi}(u) = \left\{ \frac{\psi^2 \sum_{r=1}^{n_r} \sum_{i=1}^{n_g} \mathbbm{1}(U_{ir} > u, U_{0r} > u)}{\pi \sum_{r=1}^{n_r} \mathbbm{1}(U_{0r} > u)} \right\}^{1/2}.
\end{equation}
Within each replicate, the ARE metric first evaluates the area where the grid cells jointly exceed the quantile $u$ along with the reference point $\bs_0$. This area is then converted to a radial scale using the unit side length of the grid and the factor $\pi$. Finally, the metric is averaged across all replicates to adequately represent the spatial extent of extreme events in the domain of interest. See more of the asymptotic properties of the ARE metric in \cite{zhang2023flexible}.

To measure the quality of the reconstructed spatial fields, we compute the tail-weighted continuous ranked probability score (twCRPS) \citep{gneiting2011comparing} across time for each location. For a predictive CDF $F_{it}$ and an observed value $x_{it}$ at location $\bs_i$ and time $t$, the tailed-weighted CRPS is defined as:
\begin{equation*}
    \text{twCRPS}(F_{it}, x_{it}) = \int_{-\infty}^{\infty} w(z)\{F_{it}(z)-\mathbbm{1}(z \geq x_{it})\}^2 \mathrm{d}z,
\end{equation*}
where $\mathbbm{1}(\cdot)$ is the indicator function and $w(z)$ is the weight function. To focus on upper tails, we let $w(z) = \mathbbm{1}(z > u_{90})$ and $u_{90}$ is the 90th percentile of the 2000 samples emulated at location $\bs_i$ and time $t$. The CDF $F_{it}$ is estimated empirically using the same samples. Lower tailed-weighted CRPS values indicate better model performance as they assess the squared discrepancy between the observed realizations and the predictive distribution. Examining tailed-weighted CRPS by location allows visualization of alignment between truth and emulation. Additionally, we complement the tailed-weighted CRPS with Quantile-Quantile (Q-Q) plots, which compare the empirical distribution of observed values against that of the model emulations. Under a well-specified model, the points should align closely along the $45^\circ$ reference line, indicating consistency between the observed and predicted distributions.
\section{Simulation Study and Results}\label{sec:simulation}

To demonstrate the ability of the conditional XVAE to characterize the non-stationary extremal dependence structures of large-scale spatial datasets, we simulate the data on a $50 \times 50$ regular grid within the square $[0,20]\times [0,20]$. We use the univariate El Ni\~{n}o/Southern Oscillation (ENSO) index as a climate condition (denoted by the scalar $c_t$). The ENSO index is originally reported for overlapping three-month periods from January 1980 to December 2023 including 528 time points ($n_t=528$), and is publicly accessible through the Climate Indices List of the National Oceanic and Atmospheric Administration at \href{https://psl.noaa.gov/data/climateindices/list/}{https://psl.noaa.gov/data/climateindices/list/}. Large positive values of the ENSO index indicate a strong signal of El Ni\~{n}o occurrence and large negative values correspond to La Ni\~{n}a occurrence. To obtain a smoother and more continuous representation, we further apply a centered five-month moving average to the ENSO series to smooth out year-to-year variability
Also, we normalize it to the support of $[0,1]$; see Figure~\ref{fig:ENSO} for the ENSO time series.

All experiments were performed on a desktop machine equipped with an
Intel\textsuperscript{®} Core\textsuperscript{TM} i5-9600K CPU @ 3.70\,GHz (6 cores, 6 threads) and
48\,GB of RAM. No GPU acceleration was used.

In both the simulation study and the real-world data analysis, we applied the same hyperparameter tuning strategy. To select the optimal configuration, we performed a grid search over a predefined set of candidate values. Specifically, we constructed a grid for key hyperparameters, including the learning rate, network initialization, and architectural components, and trained the model under each setting. Model performance was evaluated using the negative ELBO loss as the criterion, and the hyperparameter combination that achieved the lowest negative ELBO loss was selected as the final configuration.

\subsection{Simulation Setup}\label{subsec:sim_setup}

\begin{figure}[!ht]
\centering
    \begin{subfigure}[t]{0.99\linewidth}
        \centering
        \includegraphics[width=\linewidth]{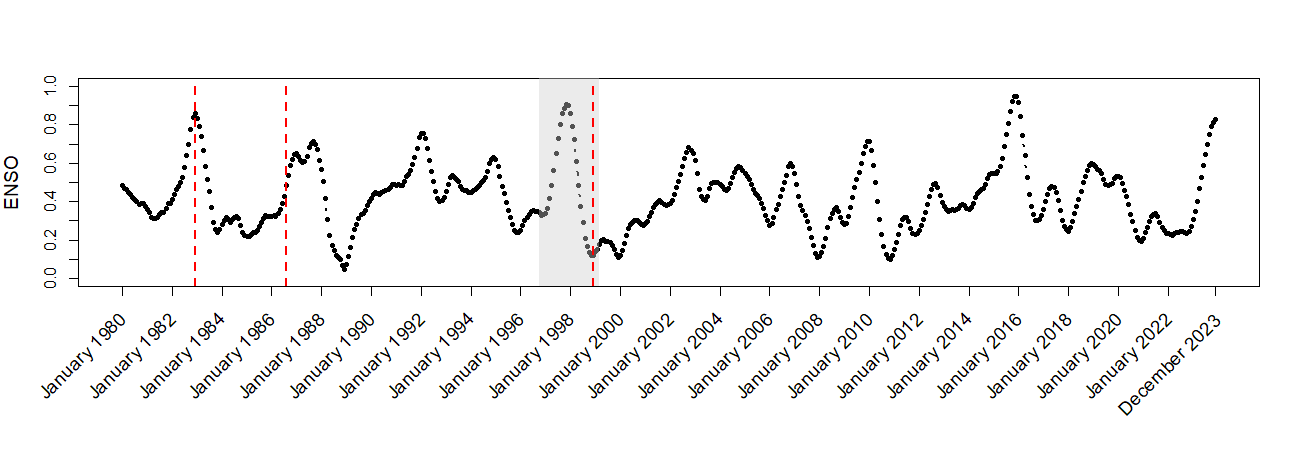}
        \caption{}
        \label{fig:ENSO}
    \end{subfigure}
    
    \vspace{0.5em}
    
    \begin{subfigure}[t]{0.32\linewidth}
        \centering
        \includegraphics[width=\linewidth]{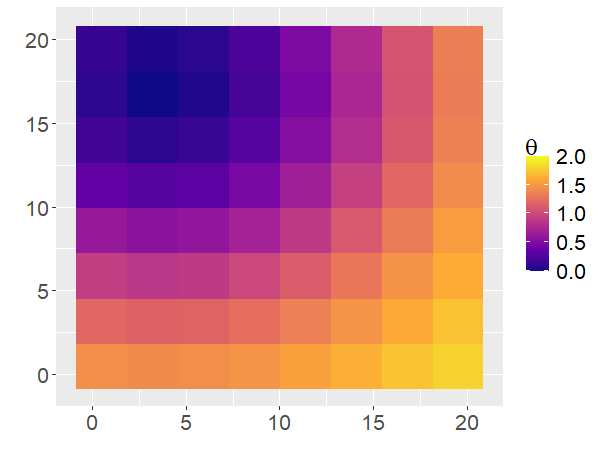}
        \caption{}
        \label{fig:2_1}
    \end{subfigure}
    \hfill
    \begin{subfigure}[t]{0.32\linewidth}
        \centering
        \includegraphics[width=\linewidth]{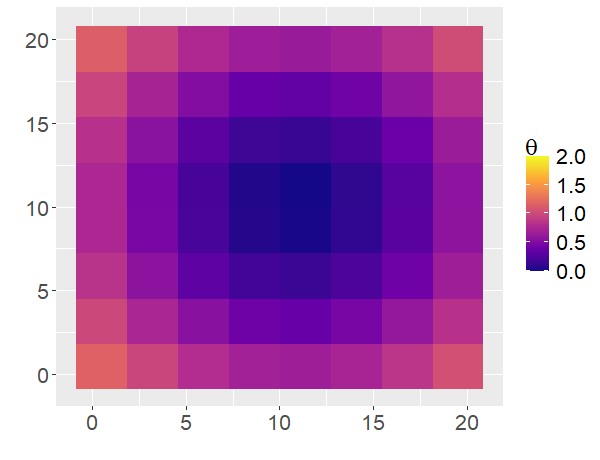}
        \caption{}
        \label{fig:2_2}
    \end{subfigure}
    \hfill
    \begin{subfigure}[t]{0.32\linewidth}
        \centering
        \includegraphics[width=\linewidth]{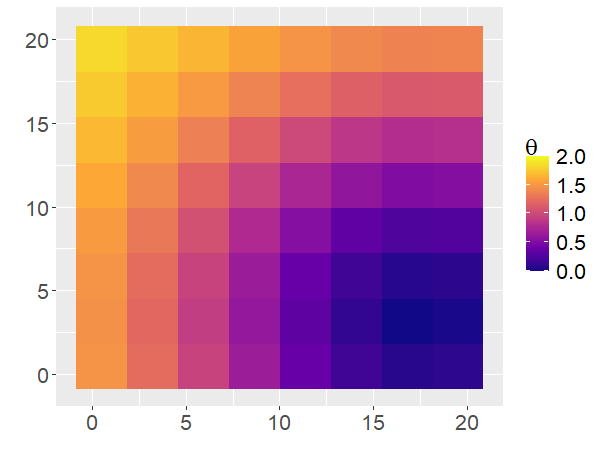}
        \caption{}
        \label{fig:2_3}
    \end{subfigure}
    
    \caption{(a): The smoothed ENSO time series $c_t$ is shown as black dots after applying a 5-month moving average to the raw ENSO time series. (b): Simulated $\btheta_t(c_t)$ when $c_t= 0.859$ in December 1982 (first red dash line). (c): Simulated $\btheta_t(c_t)$ when $c_t= 0.482$ in August 1986 (second red dash line). (d): Simulated $\btheta_t(c_t)$ when $c_t= 0.118$ in December 1998 (third red dash line).}
    \label{fig:MEI and thetas}
\end{figure}
The simulation starts with the generation of the tilting parameters in the exponentially-tilted positive stable distribution. We construct a spatial field $\btheta_t(c_t)$ that varies with the ENSO index $c_t$. Specifically,  we generate $\btheta_t(c_t)$ using a powered-exponential kernel with a fixed bandwidth and its center shifting along the off-diagonal line (between two anchor points $(0,20)^\top$ and $(20,0)^\top$):\begin{equation}\label{eqn:sim_theta}
\begin{split}
\boldsymbol{l}_t&=c_t(0,20)^\top+(1-c_t)(20,0)^\top,\\
\theta_{kt}(c_t) &= \gamma \exp \left\{ -\left(\frac{||\bg_k-\boldsymbol{l}_t||}{\tau} \right)^{b}\right\}, \quad \gamma=2,\ b=2,\ \tau=15,
\end{split}
\end{equation}
where $k=1,\ldots,K$, $\bg_k$ denotes the coordinates of the $k$th knot in the latent space and $\btheta_t(c_t)=(\theta_{1t}(c_t),\ldots, \theta_{Kt}(c_t))^\top$.
To keep the simulation realistic, the range of $\btheta_t$ is restrained to $[0, 2]$ for all times and the number of knots is set to be $K=8 \times 8=64$. The patterns of $\btheta_t$ are designed to evolve with the ENSO index values. That is, when the ENSO index $c_t$ reaches the maximum level, the center of the basis function will be located at the top-left corner, whereas at the minimum $c_t$, the center will be located at the bottom-right corner. Thus, by doing so, the center of the low tilting parameter values will move along the off-diagonal line. For the times when the ENSO index is neither large nor small, we call these ``neutral'' times, the center of basis function will wander around the middle of the map. For example, the simulated $\btheta_t(c_t)$ map when $c_t = 0.859$ in 1982-12 is shown in Figure~\ref{fig:2_1}, corresponding to relatively high  El Ni\~{n}o occurrence. The neutral time is shown in Figure~\ref{fig:2_2} with $c_t = 0.482$, and the La Ni\~{n}a time is shown in Figure~\ref{fig:2_3} with $c_t = 0.118$.



Starting from the $\btheta_t(c_t)$ values, we use the simple rejection sampler introduced in \cite{bopp2021hierarchical} to sample the latent variables $Z_{kt}$ knot-wise, with fixed $\alpha=0.5$. We follow the model in~\eqref{eqn:model_Y} to simulate the low-rank representation with $\{w_k(\bs):k=1,\ldots, K\}$ specified using the Wendland basis functions with radii of 3. The noise process at the data level that follows the log-Laplace distribution is generated with $\alpha_0 = 30$. Based on the model introduced in Section~\ref{subsec:Conditional XVAE}, we are able to sample the data $\bX_t$ for $t = 1,\ldots,n_t$.

In fitting the model, the training process converged after approximately 1000 epochs, which took about 223.06 seconds. 

\begin{figure}[htbp]
    \centering
    \includegraphics[width=0.8\linewidth]{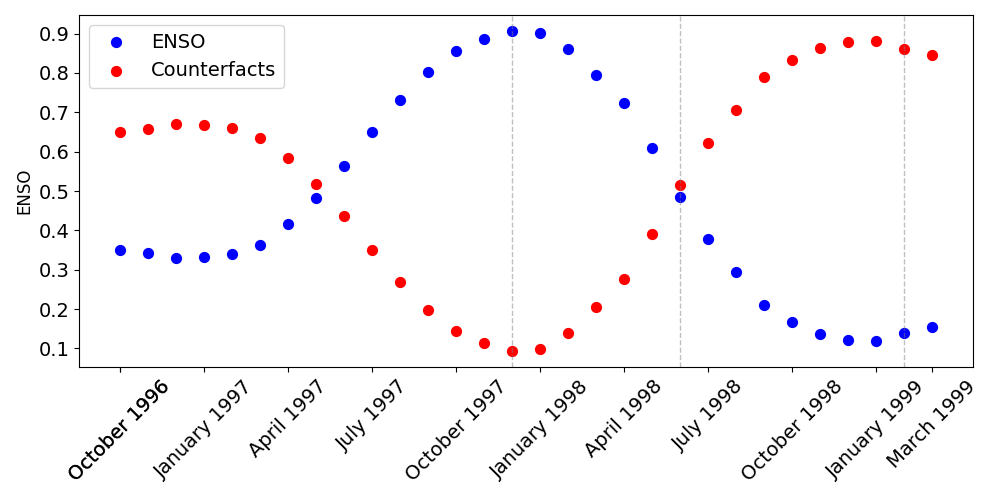}

    \vspace{0.1em}
    \includegraphics[width=0.75\linewidth]{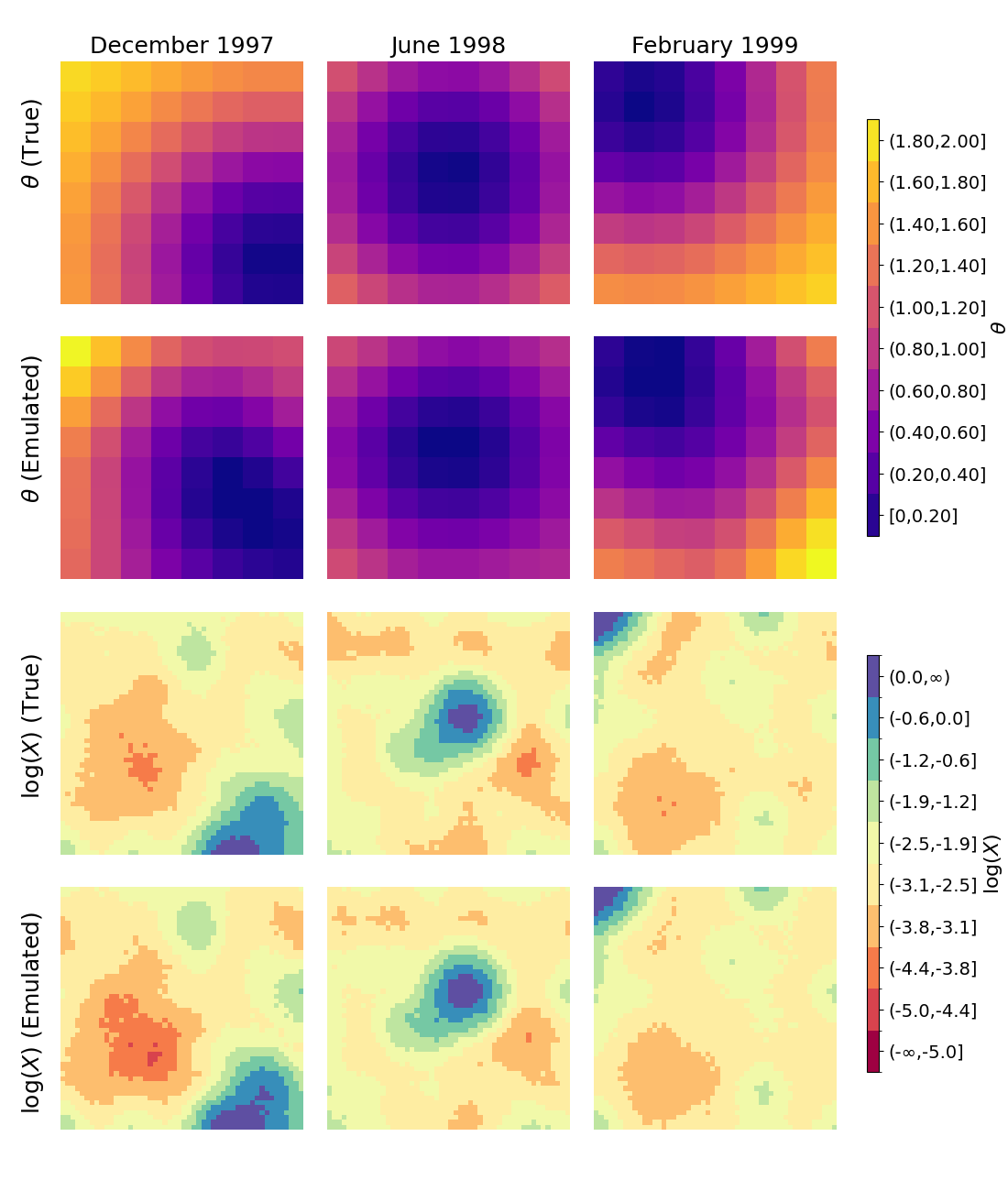}

    \caption{First row: ENSO indexes and counterfactual ENSO indexes (flipped) from October 1996 to March 1999 (the time window marked in the shades of Figure~\ref{fig:ENSO}). Second row: True $\btheta_t$ at 3 selected times. Third row: Estimated $\btheta_t$ at 3 selected times. Fourth row: True $\log(\bX_t)$ at 3 selected times. Fifth row: Emulated $\log(\bX_t)$ at 3 selected times.}
    \label{fig:sim_res}
\end{figure}

\subsection{Simulation results}

Figure~\ref{fig:sim_res} first presents a scatter plot of the ENSO index along with a synthetic ENSO index to demonstrate a counterfactual effect (see Section~\ref{sec:cf} below). The second row to the fifth row of Figure~\ref{fig:sim_res} compares the true tilting parameters $\btheta_t$ to their estimated values, and true process fields $\log (\bX_t)$ to their emulated values. These are compared at three representative time points, i.e., December 1997, June 1998, and February 1999, corresponding to an El Ni\~no, neutral, and La Ni\~na period, respectively.  Specifically, the second and third rows show the evolution of the latent variable $\btheta_t$, which governs the spatial variation in the strengths of extremal dependence over time. The fourth and fifth rows display the corresponding realizations of the physical field $\log (\bX_t)$.

The estimated $\btheta_t$ fields capture the large-scale spatial structure and smooth gradients evident in the true $\btheta_t$ fields at all time points. Note that the estimated $\btheta_t$ preserves the central low-intensity region which gradually moves from the lower right corner to the upper left, reflecting that the model successfully learned the tail heaviness of the underlying latent process. Although the estimated $\btheta_t$ is not perfectly shaped like the truth, the overall spatial patterns and intensities remain consistent with the truth.

For the physical field $\log (\bX_t)$, the emulated fields closely match the real ones, particularly in terms of spatial clustering throughout the spatial domain. The emulated maps reproduce the location of the clusters of high and low values as well as the variability in the true data, indicating the model's success in reconstructing the physical outputs from the learned latent structure via the decoder. 

Since the emulated $\btheta_t$ fields do not perfectly recover the truth, it is important to emphasize the difficulty of this task. In our results, each $\btheta_t$ estimation and physical field are constructed from only three pseudo-replicates per time point, indicating a significant challenge due to the limited available information. Typically, extremal parameter estimation is performed under stringent parametric assumptions, where time-varying parameters are modeled through functions governed by a very small set of range or scale parameters. Moreover, such estimation tasks often require much larger sample sizes due to the rarity and instability of extremes.  In contrast, our approach tackles the much more challenging problem that the amount of information is extremely insufficient to fully recover the underlying signal with precision.


In the past, most spatial extremes models assumed a single dependence parameter over the space to characterized the spatial dependence structure. For instance, within the domain of max-stable processes models, the exponent function is the key to evaluate the extremal dependence  \citep[see Section 3 in][]{huser2022advances}, which are usually simply parameterized. Similarly, the dependence class of the \citet{huser2019modeling} model is governed by a single parameter over the entire spatial domain. There are recent works where the extremal dependence parameter is allowed to be the spatially-varying similarly through low-dimensional basis representation \citep[e.g.,][]{shi2024spatial}, but their inference through traditional Bayesian MCMC is computationally prohibitive and resource intensive. Also, these works still do not allow the dependence structure to vary across time.  

Despite this challenge, our emulator is able to recover smooth latent fields and generate realistic physical output. In particular, the comparison in Figure~\ref{fig:sim_res} highlights the model's capacity to generate spatial extremes while maintaining fidelity to the smooth latent dynamics. Notably, this is achieved \textit{without assuming temporal stationarity}, further demonstrating the flexibility and robustness of the approach. Overall, the visual results further reinforce the emulator's effectiveness in preserving the spatial and temporal dynamics of extremes.

\begin{figure}[htbp]

    \begin{subfigure}[b]{0.32\linewidth}
        \centering
        \includegraphics[width=\linewidth]{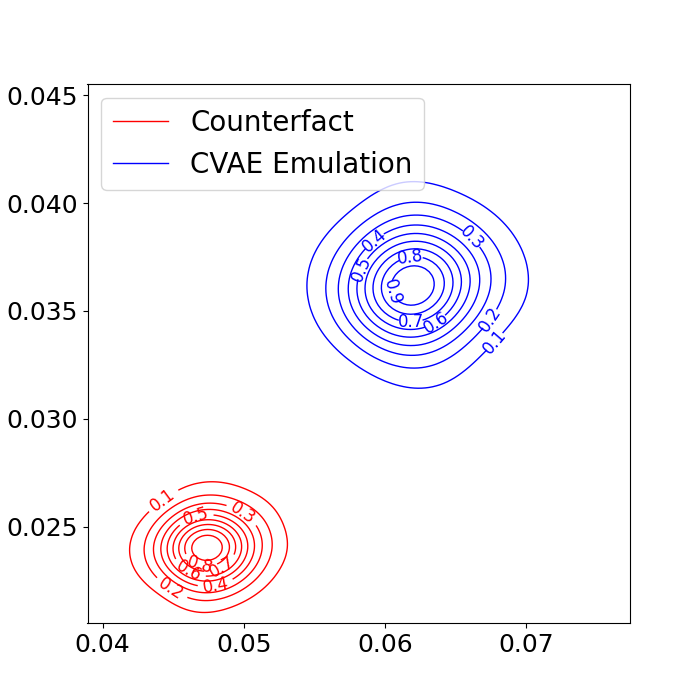}
    \end{subfigure}
    \hfill
    \begin{subfigure}[b]{0.32\linewidth}
        \centering
        \includegraphics[width=\linewidth]{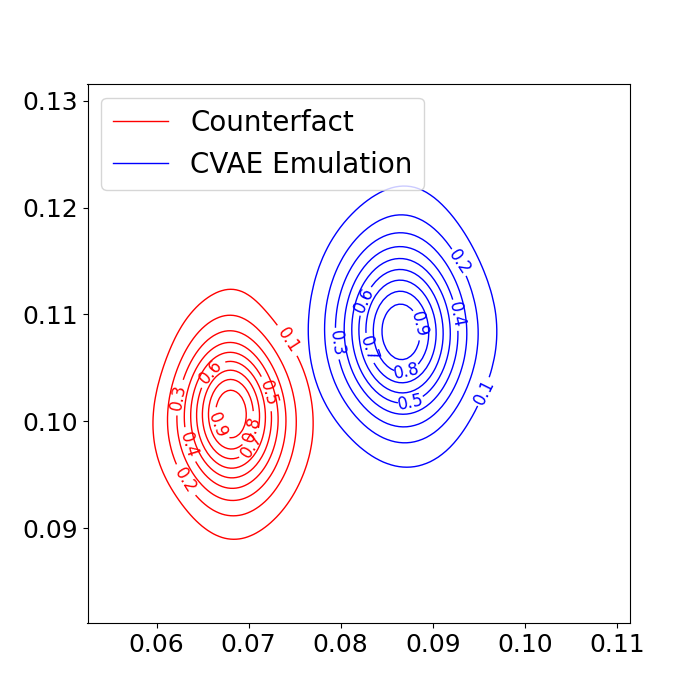}
    \end{subfigure}
    \hfill
    \begin{subfigure}[b]{0.32\linewidth}
        \centering
        \includegraphics[width=\linewidth]{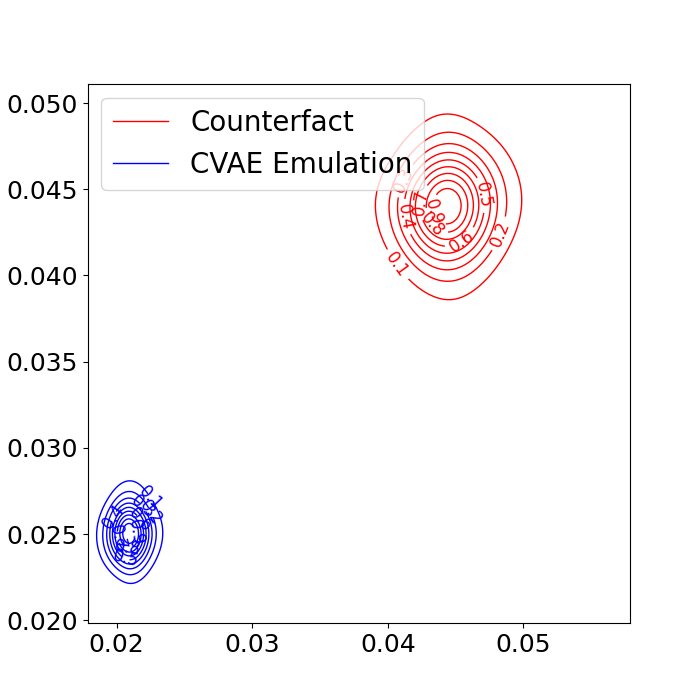}
    \end{subfigure}

    \caption{Kernel density contour plots of emulated samples at two selected spatial locations under original and counterfactual ENSO conditions. Each panel corresponds to a different time: December 1997 (left),  June 1998 (middle), and  February 1999 (right). The counterfactual ENSO signal induces clear differences in the distributions, particularly in December 1997 and February 1999, where clear deviations between the counterfactual (red) and emulation (blue) contours are observed.}
    \label{fig:sim_contour}
\end{figure}

\subsection{Counterfactual Experiment}\label{sec:cf}
To investigate the impact of pseudo-natural climate conditions on dependence parameters and reconstructed data, we manipulate the ENSO index in ``a world that might have been.'' That is, we flip its sign as illustrated in the scatter plot in Figure~\ref{fig:sim_res} and pass it through the trained conditional XVAE model. So, using the synthetic ENSO index, we can generate the counterfactual estimation of the reconstructed $\bX_t$. To better demonstrate the effect of these counterfactuals, we arbitrarily select two locations within the spatial domain and generate 500 emulated samples for each location. Then we plot the kernel density estimation of those samples in December 1997, June 1998 and February 1999, as shown in Figure~\ref{fig:sim_contour}. If the counterfactual ENSO conditions have a notable impact, we expect to observe distinct differences between the resulting contour plots. Indeed, for December 1997 and February 1999, the differences between the original and counterfactual ENSO indexes are clear. Then, we observe substantial discrepancies between the contours in the left and right panels of Figure~\ref{fig:sim_contour}. Furthermore, the direction of the contour deviations aligns with the way in which ENSO changes, indicating that the counterfactuals correctly capture an inverted temporal evolution consistent with the manipulated ENSO signal. These systematic shifts illustrate that the model does not merely reproduce observation-level changes but adjusts the latent representation in response to the altered input conditions. In this way, the model also functions as a diagnostic tool for assessing whether a conditioning variable acts as an important confounder or predictor: if a condition were uninformative, flipping its value would yield minor changes in the reconstructed distributions. In contrast, the substantial and structured differences observed here provide strong evidence that ENSO acts as the primary driver of the extremal dependence patterns by construction in this simulation experiment.

Overall, this experiment confirms that our model successfully internalizes the influence of climate conditions and can meaningfully interpolate to counterfactual climate scenarios. At the same time, the analysis validates the relevance of the conditioning variable itself within the modeling framework.

\subsection{Comparison in \texorpdfstring{$\chi$}{chi}-coefficient}
Figure~\ref{fig:sim_chi} displays the estimated $\chi$-coefficient for three representative spatial lags: short range (distance = 0.5), medium range (distance = 3) and long range (distance = 7). As mentioned in Section \ref{subsec:metric}, this $\chi$-coefficient assesses the extremal dependence between spatial locations by quantifying the probability that one site exceeds a high threshold conditional on another site also exceeding that threshold. Therefore, it serves as a valuable diagnostic tool for comparing the accuracy of the emulated fields against the true data to capture the joint extreme behavior. 

\begin{figure}[!htbp]
    \centering
    \includegraphics[width=0.85\textwidth]{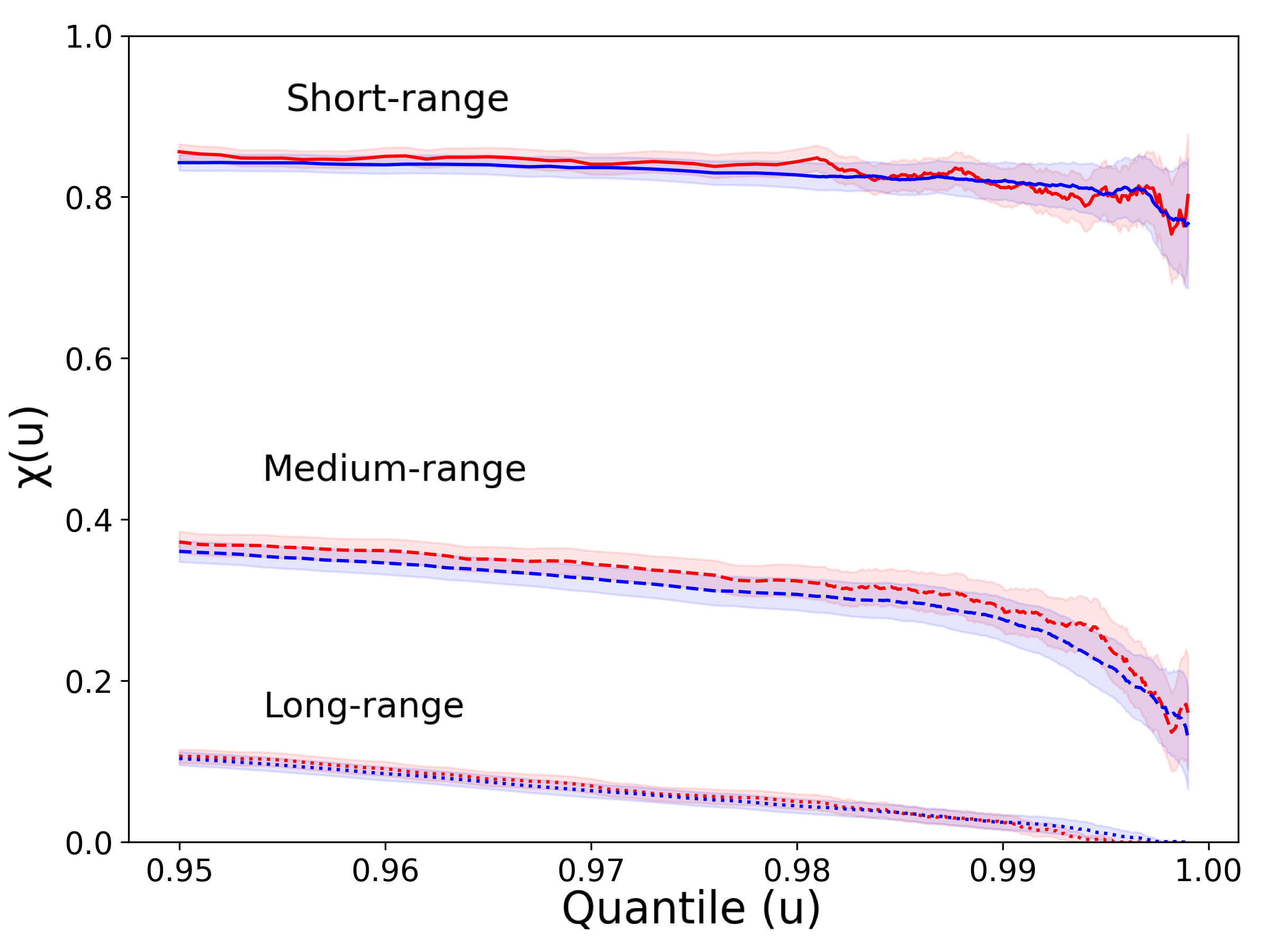}

    \caption{$\chi$-coefficients for short (distance 0.5), medium (distance 3) and long (distance 6) spatial lags. The emulated curves (blue) closely match the true data (red), capturing both strong short-range and weak long-range extremal dependence.}
    \label{fig:sim_chi}
\end{figure}

Across all distances, the emulated $\chi$ curves closely follow those of the true data. The $\chi$ value is high for the short-range pairs, which reflects the strong local dependence. The value of $\chi$ decreases rapidly as the distance increases, consistent with the trend of natural spatial dependence. Importantly, emulation maintains both the trend and the magnitude of $\chi$ on different spatial scales, including capturing the rapid drop in dependence at higher quantiles. The overlap of the 95\% confidence intervals between the truth and emulator further supports the model's ability to reproduce spatial extremal dependence with high accuracy.
Although some minor discrepancies appear at higher quantile levels and for larger distances, they remain close overall. Moreover, the 95\% confidence intervals largely overlap, indicating that the emulator still captures the weak dependence structure with reasonable accuracy.

\begin{figure}[!htbp]
    \centering
    \includegraphics[width=0.65\textwidth]{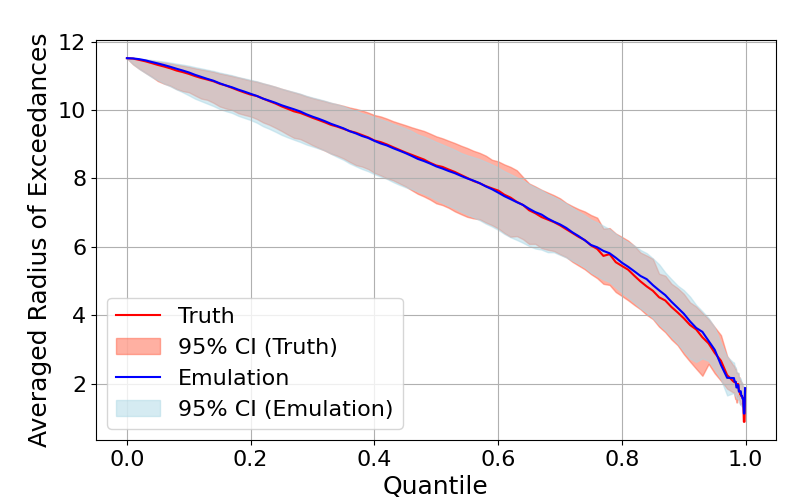}

    \caption{Comparison of ARE between ground truth data and emulated samples across quantile thresholds from 0 to 1. The ARE decreases with increasing quantiles, reflecting smaller spatial extents of extreme events. The emulated ARE closely follows the truth, with overlapping 95\% confidence intervals, indicating that the emulator accurately captures both spatial dependence and uncertainty.}
    \label{fig:sim_ARE}
\end{figure}

\subsection{Comparison in Averaged Radius of Exceedances (ARE)}
Figure~\ref{fig:sim_ARE} presents ARE curves across quantile thresholds ranging from 0 to 1, comparing the results from ground truth data with emulated samples from our emulator. As expected, ARE decreases monotonically with increasing quantile levels, which shows the decreasing spatial extent of extreme events.

The emulator exhibits a strong ability to capture the spatial dependence structure of extremes. The ARE curve generated from the emulated samples closely follows that of the truth across all quantile levels, with a negligible discrepancy. Moreover, the 95\% confidence intervals for both truth and emulation largely overlap, demonstrating that the emulator not only reproduces the expected spatial extent of extreme events but also accurately reflects the associated uncertainty.

The comparison of the ARE curves in Figure~\ref{fig:sim_ARE} supports the effectiveness of the emulator in reproducing the spatial extremes along with the bulk patterns of the spatial field. It also accurately reflects the main trend of the exceeded events, as well as the variability/uncertainty of the spatial dependence.

\begin{figure}[htbp]
    \centering
    \begin{subfigure}{0.38\textwidth}
        \includegraphics[width=\linewidth]{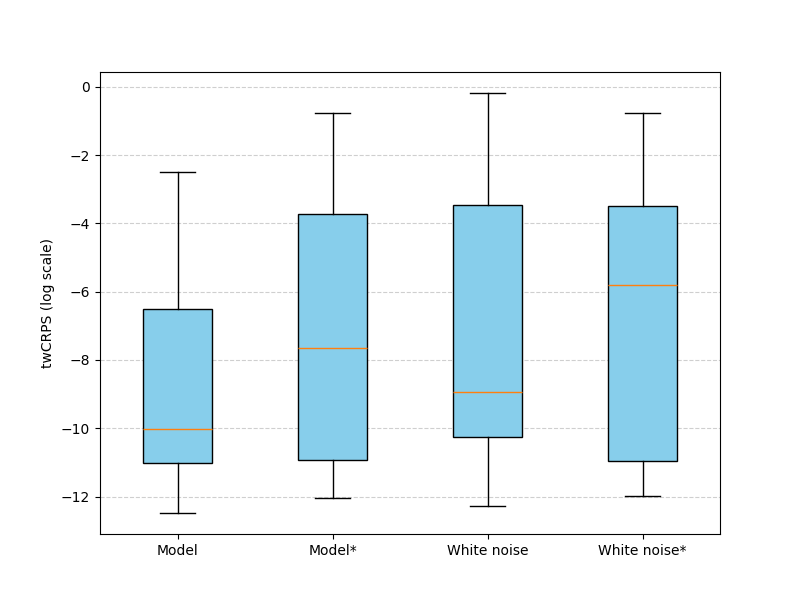}
    \end{subfigure}
    \hfill
    \begin{subfigure}{0.29\textwidth}
        \includegraphics[width=\linewidth]{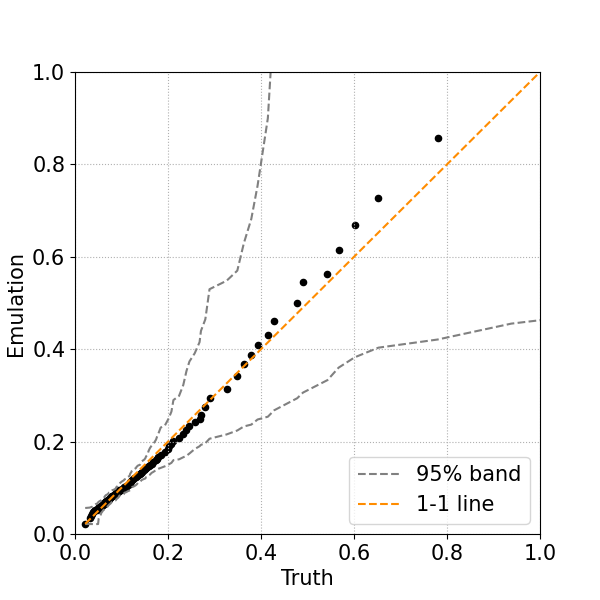}
    \end{subfigure}
    \hfill
    \begin{subfigure}{0.29\textwidth}
        \includegraphics[width=\linewidth]{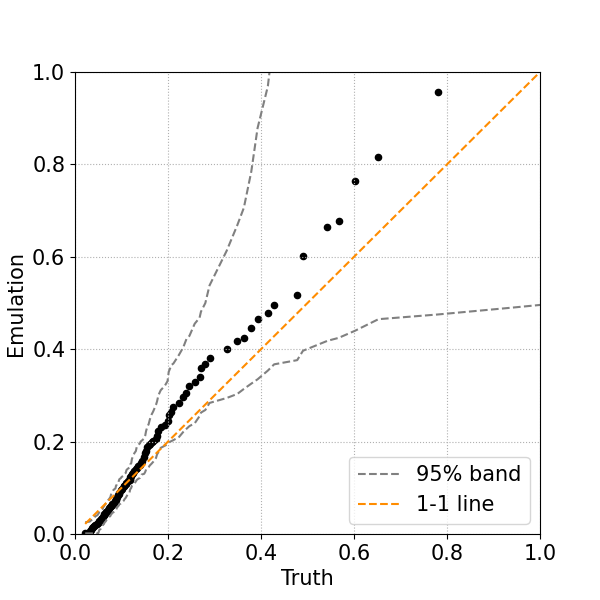}
    \end{subfigure}

    \caption{Comparison of emulator performance against truth data using tail-weighted CRPS plot and Q-Q plot. The left panel presents the boxplots of the CRPS across the holdout locations on log scale, summarizing the overall accuracy and uncertainty of the emulated fields. The first and third boxplots represent the CRPS when applying ENSO condition and white noise condition. The second and fourth boxplots represent the same comparison with fixed $\boldsymbol{W}$ in~\eqref{eqn:model_Y}. The middle panel shows the Q-Q plot between truth and emulations at a representative location, with the 1:1 line (red), indicating how well the emulator reproduces the marginal distribution of extremes. The right panel shows the Q-Q plot between truth and emulations generated with white noise conditions.}
    \label{fig:sim_CRPS}
\end{figure}
\subsection{Tail-weighted CRPS and Q-Q plot} Figure~\ref{fig:sim_CRPS} evaluates the performance of the proposed model using the tail-weighted CRPS and Q-Q plots. For the tail-weighted CRPS metric, lower values indicate better agreement between the emulated field and the truth. In the left panel, the CRPS values in the first boxplot are computed across the holdout locations using the model results and remain low among the four violins, demonstrating the high reconstruction accuracy of our model, particularly over the upper quantiles. Comparing the first and third boxplots, when the ENSO index in the model is replaced with equally scaled white noise, the tail-weighted CRPS increases substantially, highlighting the crucial contribution of the ENSO index in this experiment. Similarly, comparing the first and second boxplots shows that allowing the basis function (weight) matrix $\boldsymbol{W}$ to be learnable, rather than fixed, improves the CRPS, underscoring the necessity of adaptive spatial weighting.

In the middle panel, the Q-Q plot shows samples from a randomly selected location, with points closely following the 1:1 line, indicating that our model accurately reproduces the underlying distribution. Minor deviations are observed in the tails, where extreme values may be slightly overestimated, but these remain within an acceptable range. In contrast, the right panel shows that when samples are emulated under white noise conditions, both the bulk and tail values are clearly overestimated, again confirming the importance of incorporating ENSO conditions. Overall, both evaluations prove that the proposed model not only captures the central trend and extremes at individual locations but also provides reliable emulations across the spatial domain.


\section{Fire Weather Index Data Analysis}\label{sec:real data}

The ignition and spread of wildfires can lead to severe losses for both society and ecosystems. For example, the 2019-2020 bushfire season in Australia, often referred to as the ``Black Summer,'' burned more than 24.3 million acres of land. This catastrophic event resulted in the loss of lives, the destruction of homes, and significant damage to biodiversity and ecosystems. To investigate extremal dependence patterns in both spatial and temporal dimensions, we consider the Fire Weather Index (FWI) to access drought and fire behaviors in Australia.

\subsection{FWI Data in Australia}\label{subsec:real_data}

The FWI is part of the Canadian Forest Fire Weather Index System \citep{vanwagner1987fwi}, which evaluates the effects of weather conditions on forest floor fuel moisture. The FWI is derived from several baseline indices such as the fine fuel moisture code, the Duff moisture code, and the drought code. These baseline indices are just functions of some key weather parameters, such as temperature, relative humidity, wind speed, and precipitation \citep{dowdy2009australian}. The FWI data can be downloaded from the Global Fire Weather Database, maintained by the Goddard Earth Observing System (GEOS). This FWI dataset provides observations globally from in-land sensors and is accessible through the NASA Center for Climate Simulation (NCCS) Data Portal at  \href{https://portal.nccs.nasa.gov/datashare/GlobalFWI/}{https://portal.nccs.nasa.gov/datashare/GlobalFWI/}.

To avoid missing values in the original dataset, we extract FWI data from 1,118 locations within a target grid spanning 143.125$^{\circ}$E to 150.9375$^{\circ}$E  longitude and 33.75$^{\circ}$S to 23.25$^{\circ}$S latitude, covering the inland and coastal areas of Queensland and northern New South Wales. The grid has a resolution of 0.3125$^{\circ} \times$ 0.25$^{\circ}$. The data cover the time period from May 1, 2014 to November 30, 2024. We remove seasonality by subtracting the overall trend estimated through cubic splines, with the help of \texttt{R} package \textit{mgcv} \citep{mgcv}. At each location, we extract the monthly maxima from the detrended data, resulting in a total of 127 monthly maxima across the dataset. We then fit a GEV distribution location-wise to verify if the marginal distributions follow the GEV form. Before applying our conditional XVAE model, we use the monotonic transformations with the GEV parameters estimated for each location. The details of the data preparation procedures can be found in Appendix~\ref{App:FWI}.

\subsection{Results}\label{subsec:real_result}

The training process converged after approximately 600 epochs. The total time training time took about 981.10 seconds, after which the generation of emulated samples proceeds with minimal computational resources. That is this is an amortized inference setting where there is a reasonably high training cost but subsequent generation of emulated samples proceeds very quickly. The tuning parameters in the model were carefully selected using a cross-validation approach. 


For the domain of interest, we focus on three key seasonal snapshots: late autumn (November 2023), late spring (April 2024), and middle autumn (October 2024), corresponding to El Ni\~no, neutral, and La Ni\~na conditions, respectively. The ENSO index from April 2023 to November 2024 along with the counterfactual ENSO (flipped) are shown in the first row of Figure~\ref{fig:real_res}.  The estimated $\btheta_t$ maps are shown in the second row of Figure~\ref{fig:real_res} and reflect shifts in the heaviness of the tail of the underlying distribution, where lower $\btheta_t$ values imply a higher probability of extreme fire weather conditions. For example, the map of October 2024 indicates a particularly heavy-tailed distribution over northern New South Wales, suggesting elevated extreme event potential heading into the winter fire season.

\begin{figure}[htbp]
    \centering
    \includegraphics[width=0.75\linewidth]{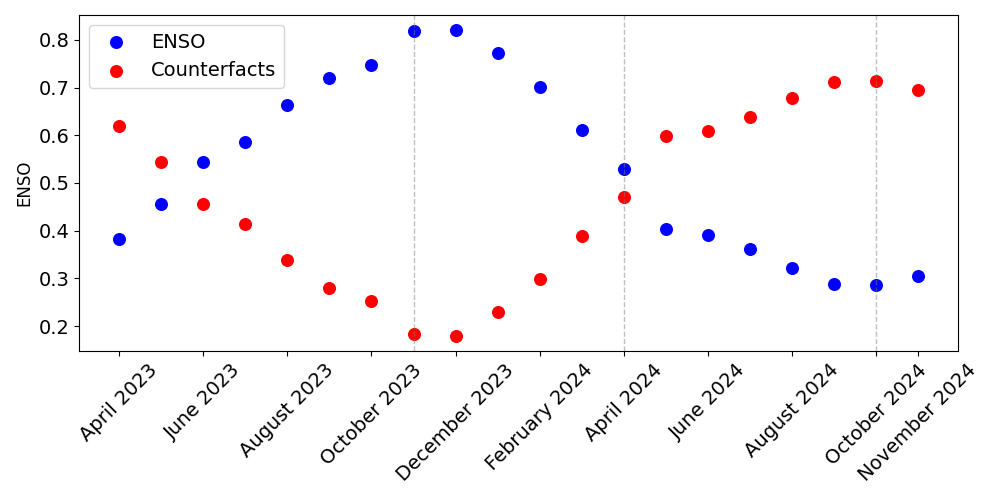}

    \vspace{0.1em}
    \includegraphics[width=0.8\linewidth]{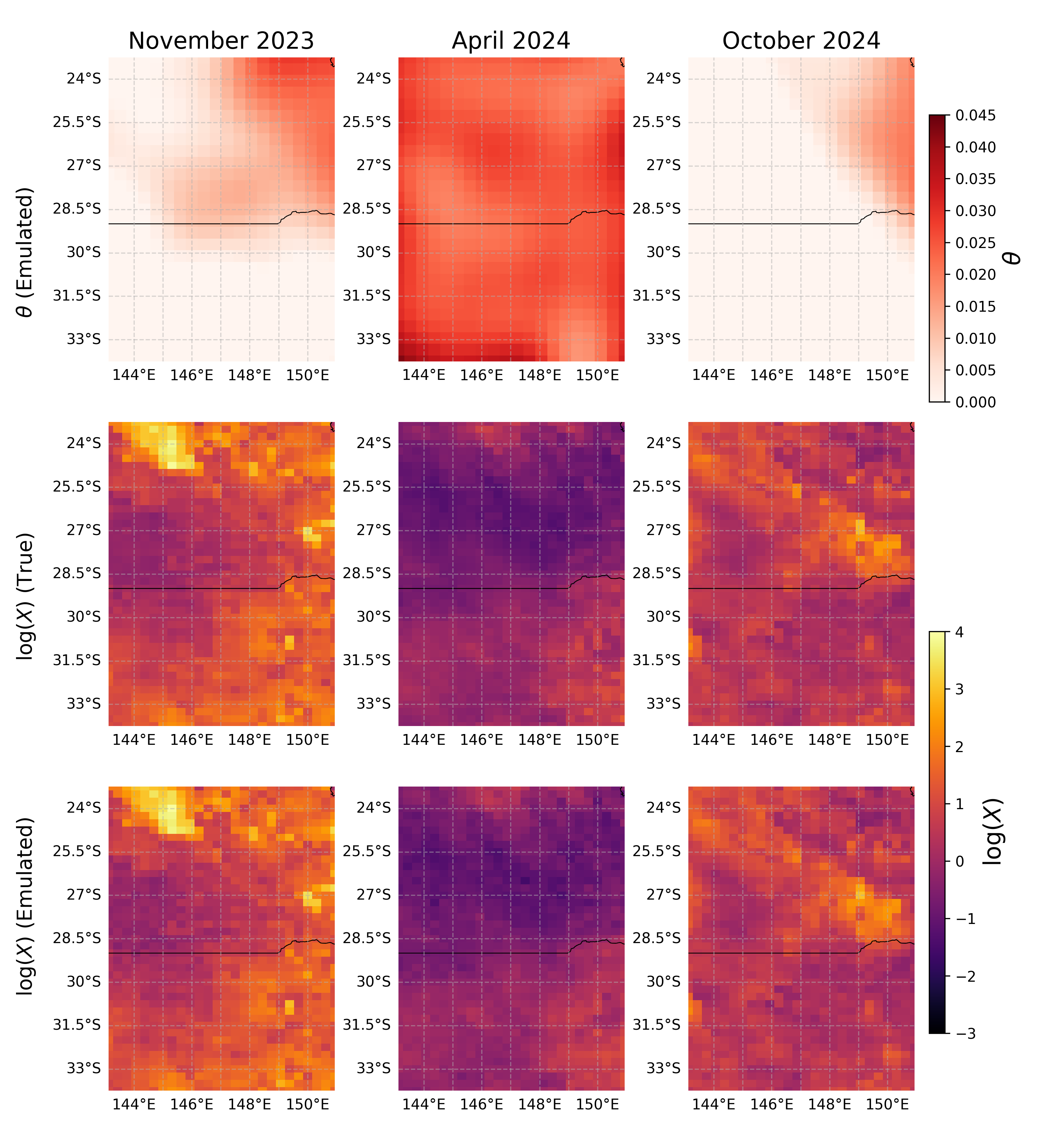}

    \caption{First row: ENSO indexes (blue) and counterfactual ENSO indexes (red) from April 2023 to November 2024. Second row: Emulated $\theta$ at 3 selected times. Third row: True $\log(X)$ at 3 selected times. Fourth row: Emulated $\log(X)$ at 3 selected times.}
    \label{fig:real_res}
\end{figure}

The true and emulated FWI fields for these times are shown in the third and fourth rows in Figure~\ref{fig:real_res}. These suggest visually that the emulated fields generated by our model closely match the ground truth in all three time periods. Our model successfully captures both large-scale seasonal trends and localized high-risk hotspots. This close agreement demonstrates the emulator's capacity to generalize across time and space, even when driven by limited pseudo-replicated fields and low-dimensional climate covariates. Importantly, we note that this exploration of our model's emulated fields and estimated extremal dependence parameters made no assumption of stationarity in space or time. 


\begin{figure}[htbp]

    \begin{subfigure}[b]{0.3\linewidth}
        \centering
        \includegraphics[width=\linewidth]{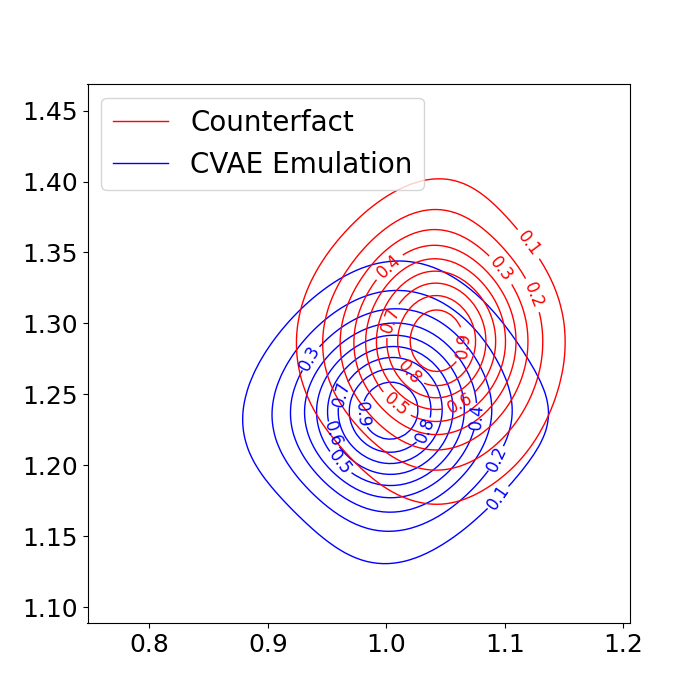}
    \end{subfigure}
    \hspace{0.01\linewidth}
    \begin{subfigure}[b]{0.3\linewidth}
        \centering
        \includegraphics[width=\linewidth]{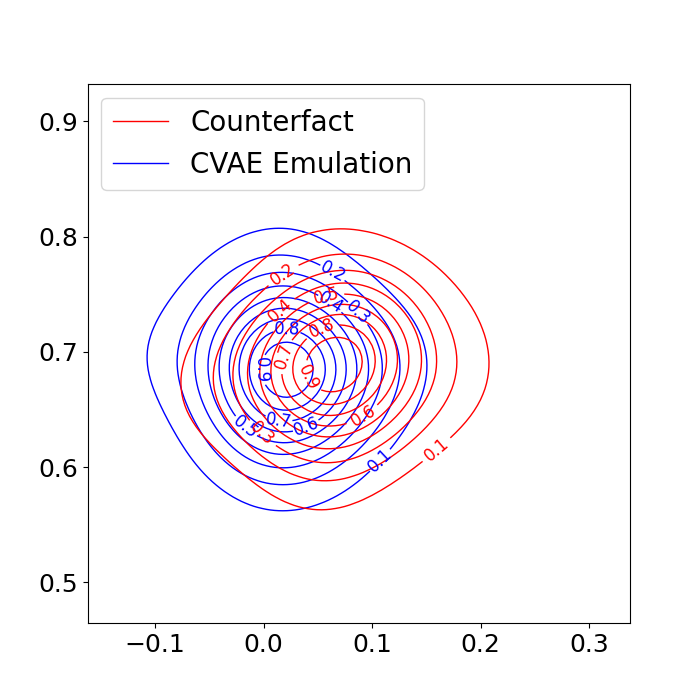}
    \end{subfigure}
    \hspace{0.01\linewidth}
    \begin{subfigure}[b]{0.3\linewidth}
        \centering
        \includegraphics[width=\linewidth]{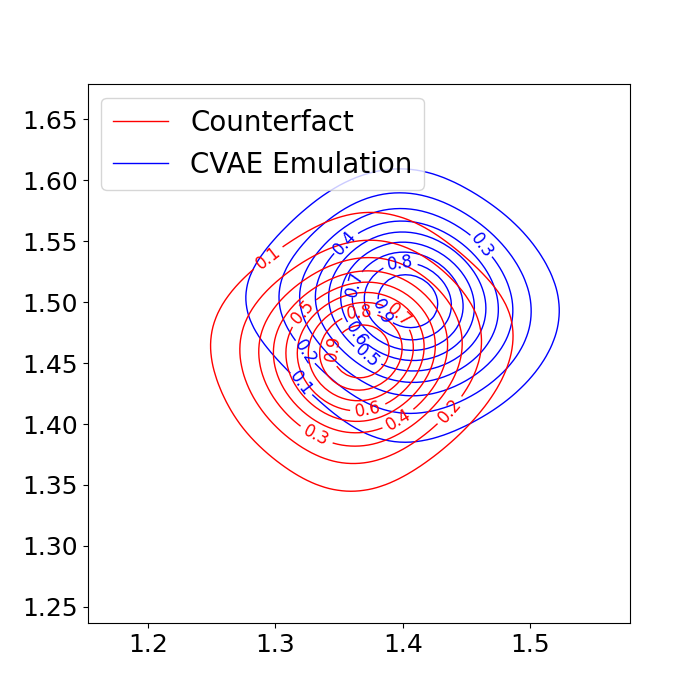}
    \end{subfigure}

    \caption{Kernel density contour plots of emulated samples at two selected spatial locations under original and counterfactual ENSO conditions. Each panel corresponds to a different time: November 2023 (left), April 2024 (middle) and October 2024 (right).}
    \label{fig:real_contour}
\end{figure}

\paragraph{Counterfactual Experiment:} As with our simulated example, we investigate how the model will react to a counterfactual ENSO time series (i.e., where it is flipped as described earlier).  Kernel density contour plots compare the distributions of the counterfactual emulations (red) versus the cXVAE emulation (blue) at the three representative times as shown in Figure~\ref{fig:real_contour}. Across all three times, the kernels of blue and red are relatively close. The time-varying structure of the counterfactual emulation is consistent with the manipulated ENSO signals; for example, the positions of the red and blue contours swap between November 2023 and October 2024 in line with the magnitude of the ENSO indices. Importantly, the minor differences between the blue and red contours indicate that the ENSO index is not a dominant predictive factor for modeling the FWI in eastern Australia. This further highlights our model’s capability to evaluate the contribution of climate drivers when modeling extremes.


Moreover, the counterfactual results provide a framework for exploring ``what-if" climate scenarios that are directly relevant for risk assessment and emergency planning. For instance, consider the case of eastern Australia during November 2023: although a low ENSO index would typically suggest reduced fire danger, the counterfactual contour plots indicate that the FWI could still reach elevated levels. This highlights the complex relationship between sea surface temperature anomalies and inland fire danger: Even when ENSO-related temperatures are cooler than usual, the region may still face a substantial risk of wildfires due, for example, to existing drought conditions.

\paragraph{Comparison in $\chi$-coefficient:}
Figure~\ref{fig:real_chi} displays the estimated conditional extremal dependence coefficient $\chi$ for three representative spatial lags: short range (distance = 2), medium range (distance = 6), and long range (distance = 10). Across all spatial lags, the emulated $\chi$ (blue curves) closely follows the true $\chi$ from the data (red curves). At short range, the $\chi$ values are relatively high, reflecting a moderate spatial dependence of the extremes. The CXAVE model captures this structure well, preserving the overall shape and the steep decline of $\chi$ as the quantile increases. At medium-range and long-range, the $\chi$ values decrease as expected, representing the decay of dependence with distance. The cXVAE model continues to track the trend with reasonable accuracy, along with acceptable underestimation or overestimation at some high quantile levels. The 95\% confidence intervals from the cXVAE model (shaded blue) and from the truth (shaded red) generally overlap, especially in the lower and intermediate quantile ranges. This overlap supports the emulator's ability to replicate the correct dependence structure.

\begin{figure}[htbp]
    \centering
\includegraphics[width=0.75\linewidth]{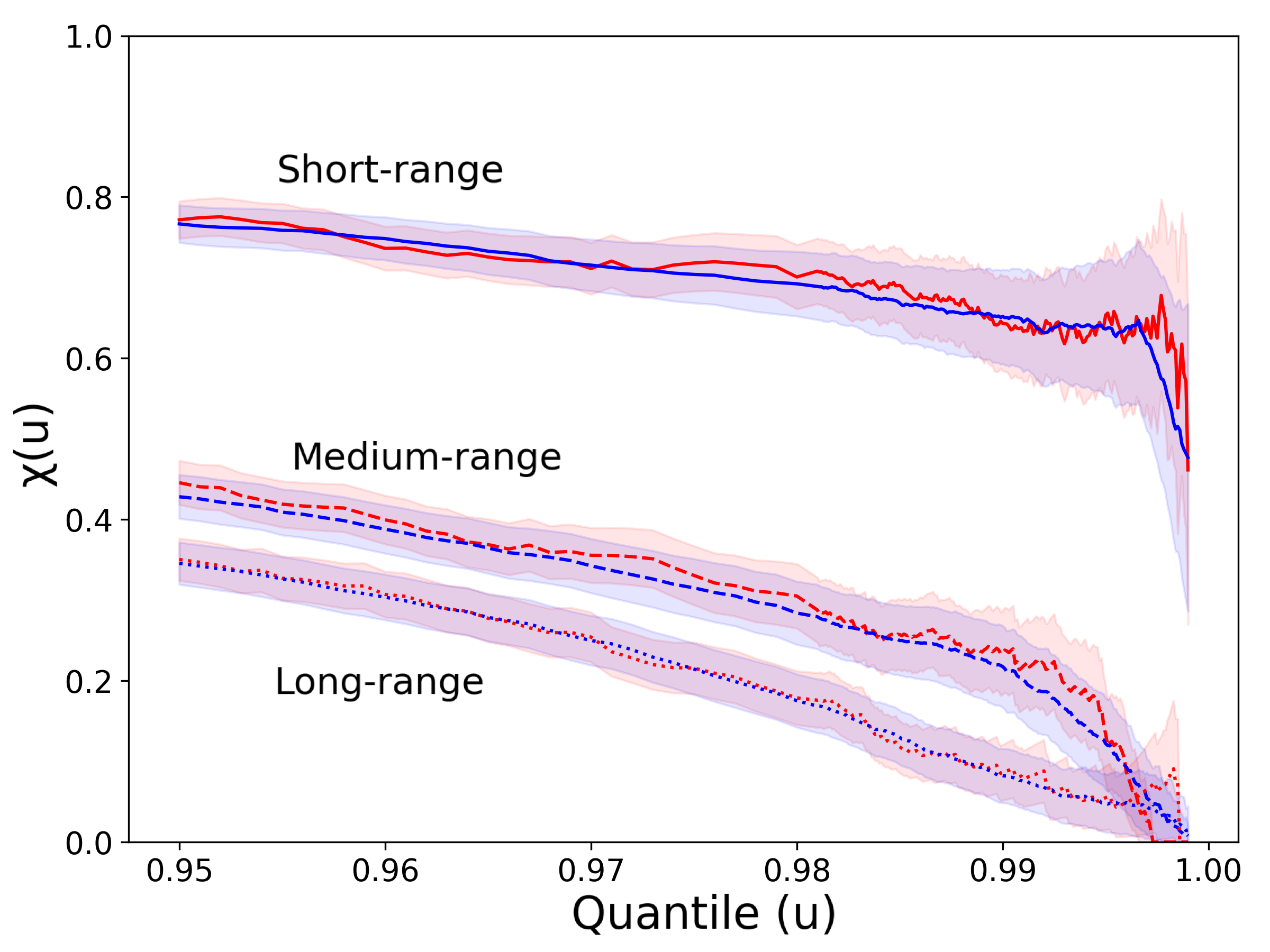}
    \caption{$\chi$-coefficients for short (distance 2), medium (distance 6) and long (distance 10) spatial lags. The emulated curves (blue) closely match the true data (red).}
    \label{fig:real_chi}
\end{figure}

\paragraph{Comparison in Averaged Radius of Exceedances (ARE):}
Figure~\ref{fig:real_ARE} presents the ARE curves in a sequence of quantile thresholds, comparing the spatial extent of extreme events in the observed (truth) FWI data with those generated by our model. The cXVAE model demonstrates strong performance in replicating the spatial extent of extremes. The ARE curves calculated from the emulated samples closely track those from the observed data across nearly the entire quantile range. Moreover, the 95\% confidence intervals for the emulated data largely overlap with those of the observed data, indicating that the CXAVE model not only captures the mean spatial extent of extreme events, but also reproduces the associated variability with high accuracy.

\begin{figure}[!t]
    \centering
    \includegraphics[width=0.7\textwidth]{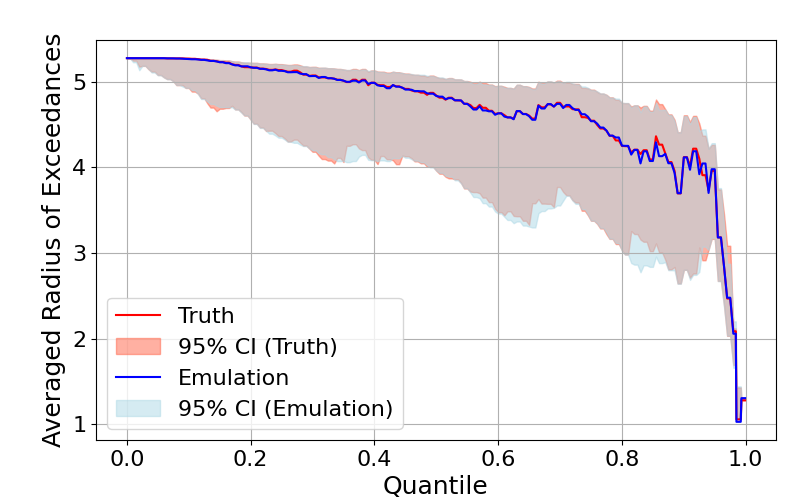}

    \caption{Comparison of ARE between the observed FWI and emulated FWI samples across quantile thresholds from 0 to 1.}
    \label{fig:real_ARE}
\end{figure}

\paragraph{Tail-weighted CRPS and Q-Q plot:} The results in Figure~\ref{fig:real_CRPS} demonstrate that the cXVAE model provides accurate and reliable reconstructions for the majority of cases. In the left panel, the tail-weighted CRPS distribution exhibits low variability, and the low CRPS values suggest sound model performance across the dataset. Across all four boxplots, our model slightly outperforms the alternatives, which is consistent with the limited influence of ENSO conditions in this experiment. In the middle panel, the Q–Q plot shows that our model closely reproduces the true FWI across all quantiles. In the right panel, the Q–Q plot based on white-noise conditions remains highly consistent with the truth, providing further evidence that ENSO contributes only modestly in this setting. Taken together, both diagnostics confirm that the proposed model provides an effective framework for generating emulated samples, and moreover, for identifying the relevance of climate conditions in modeling spatial extremes.

\begin{figure}[htbp]
    \centering
    \begin{subfigure}{0.38\textwidth}
        \includegraphics[width=\linewidth]{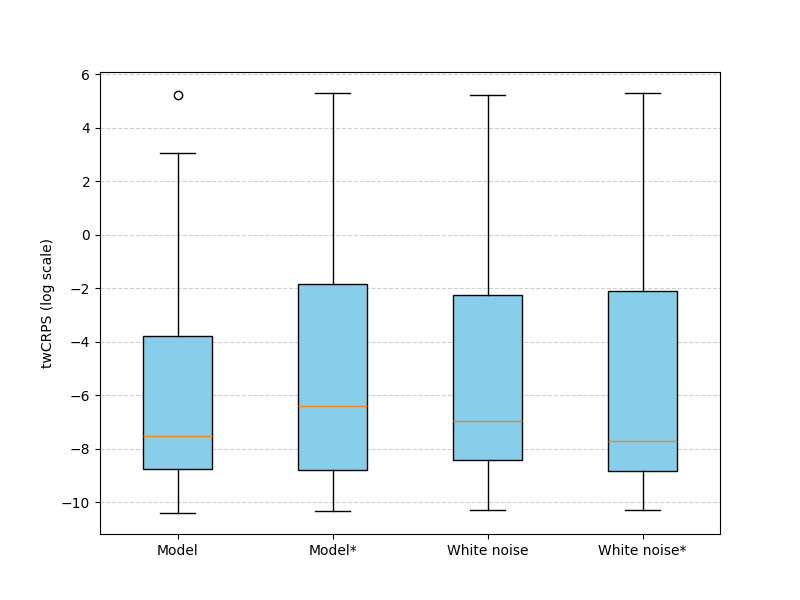}
    \end{subfigure}
    \hfill
    \begin{subfigure}{0.29\textwidth}
        \includegraphics[width=\linewidth]{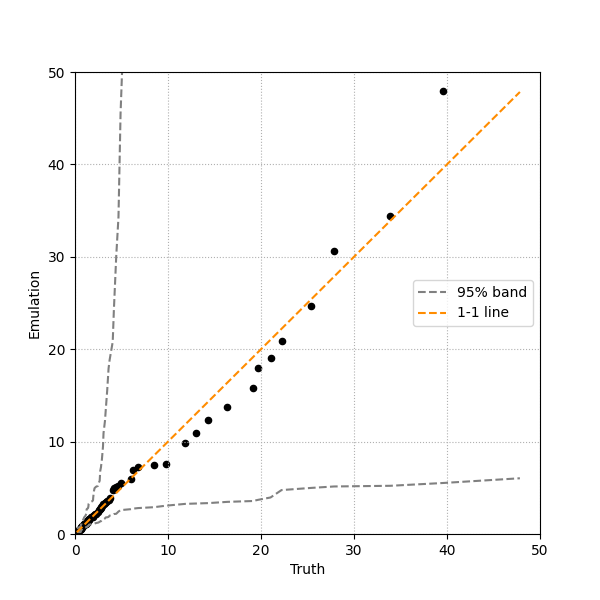}
    \end{subfigure}
    \hfill
    \begin{subfigure}{0.29\textwidth}
        \includegraphics[width=\linewidth]{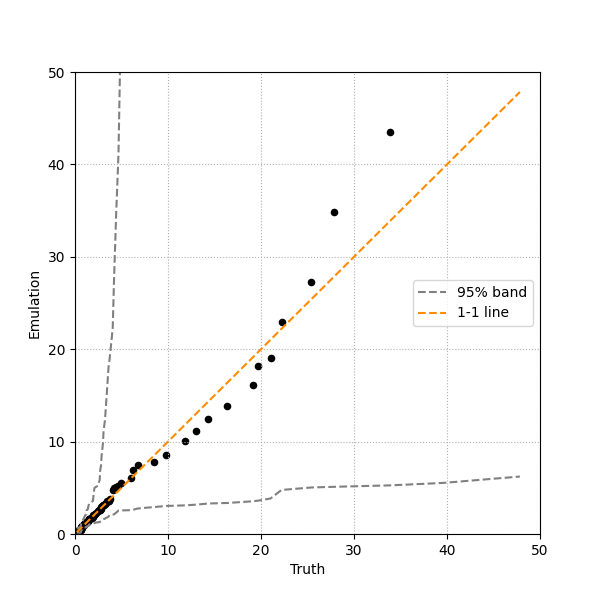}
    \end{subfigure}

    \caption{Comparison of emulated FWI sample against observed FWI using tail-weighted CRPS plot and Q-Q plot. The left panel presents the boxplots of the CRPS across the holdout locations on log scale. The first and third boxplots represent the CRPS when applying ENSO condition and white noise condition. The second and fourth boxplots represent the same comparison with fixed $\boldsymbol{W}$ in~\eqref{eqn:model_Y}. The middle panel shows the Q-Q plot between true FWI and emulated FWI sample at a representative location, with the 1:1 line (red), indicating how well the emulator reproduces the marginal distribution of extremes. The right panel shows the Q-Q plot between true FWI and emulated FWI sample generated with white noise conditions.}
    \label{fig:real_CRPS}
\end{figure}

Overall, the results presented here suggest that the cXVAE model can model real-world time-varying extremal dependence in the spatial domain. These results demonstrate that the model effectively preserves both the spatial scale and the uncertainty of extreme events. In addition, the emulator is capable of supporting downstream applications such as climate impact assessments, hazard risk mapping, and predictive fire weather forecasting given its ability to use counterfactual scenarios to evaluate the importance of conditioning time series on the generative ability of the model.


\section{Discussion}\label{sec:discussion}

This study introduces a cXVAE model that integrates climate drivers into a deep generative framework for spatio-temporal extremes. By allowing the latent extremal-dependence parameters to change along with the climate conditions, the model moves beyond stationarity assumptions and provides a flexible tool for reconstructing extreme events under different climate states. 

The cXVAE model supports counterfactual experiments, allowing researchers to examine how extreme events might change under altered climate signals, while require only modest computational resources. In addition, by comparing emulations generated with and without conditioning variables, the model offers a diagnostic for assessing the relevance of climate drivers in explaining extremes.

One limitation of the current work arises from enforcing temporal continuity in the model by adding a penalty term that encourages smooth evolution across neighboring time points. This regularization reflects the belief that, in many real-world systems, changes are typically gradual rather than abrupt, so large jumps in the latent dynamics are unlikely. In the current implementation, we apply this penalty uniformly over all time steps, which simplifies optimization but may be suboptimal when the underlying process exhibits nonstationary behavior. For example, during rapid regime shifts a uniform penalty may overly constrain the model, while in quiet periods it may be unnecessarily strong. A more flexible alternative would allow the penalty strength to vary over time---for instance, by first estimating a ``velocity'' or rate-of-change metric from the data and using it to construct a time-varying regularization schedule. Although we did not pursue such adaptive schemes here, we view them, along with other approaches (e.g., within block-independence frameworks), as promising directions for future work on modeling temporal continuity.

Several avenues exist for extending the cXVAE. In real applications, climate and environmental variables often arise from disparate measurement systems with non-consistent spatial grids, irregular sampling, or multi-resolution structure. Accommodating such spatial misalignment would allow the cXVAE to fuse high-dimensional climate conditions. Extending the model architecture to incorporate multi-resolution or multi-index spatial representations would significantly broaden its applicability to real-world problems where consistent spatial domains cannot be assumed. These developments would enhance the practical utility of the cXVAE and broaden its relevance for real-world spatial extreme analysis.

\bibliographystyle{apalike}
\bibliography{bib.bib} 

\newpage
\appendix
\setcounter{figure}{0} 
\renewcommand{\thefigure}{A.\arabic{figure}} 

\section{Log-Laplace measurement error}\label{App:log-Laplace}
As we understand from Expression~\eqref{eqn:frechet}, the tail decay rate of the Fr\'{e}chet$(0,\tau,\alpha_0)$ distribution is Pareto-like (i.e., regularly varying):
\begin{equation*}
    \mathbb{P}\{\epsilon(\bs) > x\}
    = 1 - \exp\left\{-\left(\frac{\tau}{x}\right)^{\alpha_0}\right\}\sim \tau^{\alpha_0}\,x^{-\alpha_0},
    \qquad x \to \infty.
\end{equation*}
In this subsection, our objective is to devise a flexible error model that mirrors this tail decay rate while being concentrated around 1 to mimic the standard normal error in additive models; see Figure~\ref{fig:app_loglaplace}. Consider $U \sim \text{Laplace}(0,1/\alpha_0)$, 
with the distribution function,
\begin{align*}
    \mathbb{P}(U \le u) =
\begin{cases}
\frac{1}{2}\exp\!\left(\alpha_0{u}\right), & u < 0, \\
1 - \frac{1}{2}\exp\!\left(-\alpha_0{u}\right), & u \ge 0.
\end{cases}
\end{align*}

Then a Log-Laplace$(0, 1/\alpha_0)$ variable can be constructed by defining $\epsilon = e^U$, whose distribution function is:
\begin{align*}
    \mathbb{P}(\epsilon\leq x)=     \begin{cases}
    \frac{1}{2}x^{\alpha_0}, \quad &0 < x \leq 1, \\
    1-\frac{1}{2}x^{-\alpha_0}, \quad &x > 1,
    \end{cases}
\end{align*}
Therefore the tail $\mathbb{P}(\epsilon> x)= \frac{1}{2}x^{-\alpha_0}$, whose decay is of the same order as Fr\'{e}chet$(0,\tau,\alpha_0)$.

\section{Technical derivations}\label{sec:proofs}
To show the tail equivalence under noise replacement, we use standard results from regular variation \citep[Karamata theory and Potter bounds; see, e.g.,][Proposition 0.8]{resnick2008extreme} and the dominated convergence theorem (DCT).

\begin{lemma}[Potter bounds for regularly varying tails]
\label{lem:potter}
Let $\bar F$ be a regularly varying tail with index $-\alpha_0<0$, i.e.
\begin{equation*}
\bar F(x) = x^{-\alpha_0} L(x), \qquad x>0,
\end{equation*}
where $L$ is slowly varying at infinity. Then:

\begin{enumerate}[(1)]
\item\label{exist} For every $\delta>0$, there exist $x_0>0$ and $C>0$ such that for all $x\ge x_0$ and all $t\geq 1$,
\begin{equation*}
C^{-1} t^{-\alpha_0-\delta}
\;\le\;
\frac{\bar F(tx)}{\bar F(x)}
\;\le\;
C\, t^{-\alpha_0+\delta}.
\end{equation*}

\item In particular, for every $\delta>0$ there exist $x_0>0$ and $C_\delta>0$ such that for all $x\ge x_0$ and all $y>0$,
\begin{equation*}
\frac{\bar F(x/y)}{\bar F(x)}
\;\le\;
C_\delta\bigl(y^{\alpha_0+\delta} + y^{\alpha_0-\delta}\bigr).
\end{equation*}
Then we can find another constant $C'_\delta>0$ such that
\begin{equation}\label{eq:potter-tail-ratio}
    \frac{\bar F(x/y)}{\bar F(x)}
\;\le\;
C'_\delta\bigl(1 + y^{\alpha_0+\delta}\bigr),\; \text{for all }y>0.
\end{equation}
\end{enumerate}
\end{lemma}
The inequality in~\eqref{eq:potter-tail-ratio} follows directly from the first item in Lemma~\ref{lem:potter}. This bound provides the integrable domination needed to apply the DCT in the following proof.

\begin{figure}
    \centering
    \includegraphics[width=0.5\linewidth]{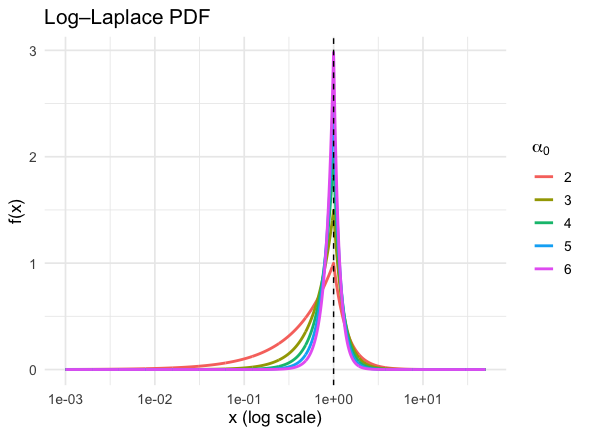}
    \caption{Log-laplace PDF under different $\alpha_0$: larger $\alpha_0$ gives us lighter tails.}
    \label{fig:app_loglaplace}
\end{figure}

\begin{proof}[Proof of Theorem~\ref{thm:tail_equiv}]


\medskip
First we look at the \textit{marginal} tail equivalence.
Fix $\bs\in\mathcal{S}$. For $X_F(\bs) = \epsilon_F(\bs)\,Y(\bs)$, independence of $\epsilon_F$ and $Y$ yields
\begin{equation*}
\mathbb{P}\{X_F(\bs)>x\}
= \mathbb{P}\{\epsilon_F(\bs) Y(\bs) > x\}
= \mathbb{E}\bigl[\mathbb{P}\{\epsilon_F(\bs) > x/Y(\bs)\mid Y(\bs)\}\bigr]
= \mathbb{E}\bigl[\bar F_{\epsilon_F}(x/Y(\bs))\bigr].
\end{equation*}
Thus,
\begin{equation*}
\frac{\mathbb{P}\{X_F(\bs)>x\}}{\bar F_{\epsilon_F}(x)}
= \mathbb{E}\!\left[
  \frac{\bar F_{\epsilon_F}(x/Y(\bs))}{\bar F_{\epsilon_F}(x)}
\right].
\end{equation*}

For each fixed $y>0$, by \eqref{eqn:tail_cond},
\begin{equation*}
\frac{\bar F_{\epsilon_F}(x/y)}{\bar F_{\epsilon_F}(x)} \rightarrow y^{\alpha_0},
\quad\text{as }x\to\infty.
\end{equation*}
Hence, pointwise in $Y(\bs)$,
\begin{equation*}
\frac{\bar F_{\epsilon_F}(x/Y(\bs))}{\bar F_{\epsilon_F}(x)}
\rightarrow Y(\bs)^{\alpha_0},\qquad \text{as }x\to\infty.
\end{equation*}

To apply the Dominated Convergence Theorem (DCT), we first choose $\delta\in(0,\eta)$ and apply the Potter bound \eqref{eq:potter-tail-ratio} for a certain constant $C'_\delta>0$:
for all sufficiently large $x$, 
\begin{equation*}
\frac{\bar F_{\epsilon_F}(x/Y(\bs))}{\bar F_{\epsilon_F}(x)}
\le C'_\delta\bigl(1+Y(\bs)^{\alpha_0+\delta}\bigr),
\end{equation*}
and our assumption $\mathbb{E}\{Y(\bs)^{\alpha_0+\eta}\}<\infty$ implies
$\mathbb{E}\{Y(\bs)^{\alpha_0+\delta}\}<\infty$, so the right-hand side is integrable. Therefore, by dominated convergence,
\begin{equation*}
\frac{\mathbb{P}\{X_F(\bs)>x\}}{\bar F_{\epsilon_F}(x)}
\rightarrow \mathbb{E}\{ Y(\bs)^{\alpha_0}\},\qquad \text{as }x\to\infty.
\end{equation*}
Since $\bar F_{\epsilon_F}(x)\sim c_F x^{-\alpha_0}$, we obtain
\begin{equation*}
\mathbb{P}\{X_F(\bs)>x\} \sim c_F\,\mathbb{E}\{Y(\bs)^{\alpha_0}\}\,x^{-\alpha_0},\qquad \text{as }x\to\infty.
\end{equation*}

Exactly the same argument applied to $X_L(\bs)=\epsilon_L(\bs)\,Y(\bs)$ with $\bar F_{\epsilon_L}(x)\sim c_Lx^{-\alpha_0}$ gives
\begin{equation*}
\mathbb{P}\{X_L(\bs)>x\} \sim c_L\,\mathbb{E}\{Y(\bs)^{\alpha_0}\}\,x^{-\alpha_0},\qquad \text{as }x\to\infty.
\end{equation*}
Hence,
\begin{equation*}
\frac{\bar F_{X_F(\bs)}(x)}{\bar F_{X_L(\bs)}(x)}
\rightarrow \frac{c_F}{c_L},\qquad \text{as }x\to\infty,
\end{equation*}
which is equivalent to
\begin{equation*}
\bar F_{X_F(\bs)}(x)\ \sim\ \frac{c_F}{c_L}\,\bar F_{X_L(\bs)}(x),\qquad \text{as }x\to\infty.
\end{equation*}

\medskip
Next we examine the bivariate \textit{joint} tail.
Now fix $\bs_1,\bs_2\in\mathcal{S}$, and we have
\begin{equation*}
X_F(\bs_i)=\epsilon_F(\bs_i)\,Y(\bs_i),\qquad i=1,2.
\end{equation*}
Condition on $(Y(\bs_1),Y(\bs_2))$, we have
\begin{equation*}
\begin{aligned}
\mathbb{P}\{X_F(\bs_1)>x,\ X_F(\bs_2)>x \mid Y(\bs_1),Y(\bs_2)\}
&= \mathbb{P}\{\epsilon_F(\bs_1)>\tfrac{x}{Y(\bs_1)},\ \epsilon_F(\bs_2)>\tfrac{x}{Y(\bs_2)}\mid Y\}\\
&= \mathbb{P}\{\epsilon_F>\tfrac{x}{Y(\bs_1)}\}\,\mathbb{P}\{\epsilon_F>\tfrac{x}{Y(\bs_2)}\}\\
&= \bar F_{\epsilon_F}\Bigl(\frac{x}{Y(\bs_1)}\Bigr)\,\bar F_{\epsilon_F}\Bigl(\frac{x}{Y(\bs_2)}\Bigr).
\end{aligned}
\end{equation*}
Therefore,
\begin{equation*}
\mathbb{P}\{X_F(\bs_1)>x,\ X_F(\bs_2)>x\}
= \mathbb{E}\!\left[
  \bar F_{\epsilon_F}\Bigl(\frac{x}{Y(\bs_1)}\Bigr)
  \bar F_{\epsilon_F}\Bigl(\frac{x}{Y(\bs_2)}\Bigr)
\right].
\end{equation*}

Divide by $\bar F_{\epsilon_F}(x)^2$:
\begin{equation*}
\frac{\mathbb{P}\{X_F(\bs_1)>x,\ X_F(\bs_2)>x\}}{\bar F_{\epsilon_F}(x)^2}
= \mathbb{E}\!\left[
  \frac{\bar F_{\epsilon_F}(x/Y(\bs_1))}{\bar F_{\epsilon_F}(x)}
  \cdot
  \frac{\bar F_{\epsilon_F}(x/Y(\bs_2))}{\bar F_{\epsilon_F}(x)}
\right].
\end{equation*}

For each fixed $(y_1,y_2)$ with $y_1,y_2>0$, \eqref{eqn:tail_cond} implies
\begin{equation*}
\frac{\bar F_{\epsilon_F}(x/y_1)}{\bar F_{\epsilon_F}(x)} \to y_1^{\alpha_0},
\quad
\frac{\bar F_{\epsilon_F}(x/y_2)}{\bar F_{\epsilon_F}(x)} \to y_2^{\alpha_0},
\end{equation*}
so pointwise,
\begin{equation*}
\frac{\bar F_{\epsilon_F}(x/Y(\bs_1))}{\bar F_{\epsilon_F}(x)}
  \cdot
\frac{\bar F_{\epsilon_F}(x/Y(\bs_2))}{\bar F_{\epsilon_F}(x)}
\rightarrow
Y(\bs_1)^{\alpha_0}Y(\bs_2)^{\alpha_0},
\qquad \text{as }x\to\infty.
\end{equation*}

To justify dominated convergence, apply the Potter bound \eqref{eq:potter-tail-ratio} twice with some $\delta>0$. For large $x$,
\begin{equation*}
\frac{\bar F_{\epsilon_F}(x/Y(\bs_i))}{\bar F_{\epsilon_F}(x)}
\le C'_\delta\bigl(1+Y(\bs_i)^{\alpha_0+\delta}\bigr),
\quad i=1,2,
\end{equation*}
so their product is bounded by
\begin{equation*}
C_\delta''\bigl(1+Y(\bs_1)^{\alpha_0+\delta}\bigr)\bigl(1+Y(\bs_2)^{\alpha_0+\delta}\bigr)
\le C_\delta''\bigl(1+Y(\bs_1)^{\alpha_0+\delta}Y(\bs_2)^{\alpha_0+\delta}\bigr).
\end{equation*}
By assumption, $\mathbb{E}\{Y(\bs_1)^{\alpha_0}Y(\bs_2)^{\alpha_0}\}<\infty$ and
$\mathbb{E}\{Y(\bs)^{\alpha_0+\eta}\}<\infty$ for each $\bs$. Choose $\delta\in(0,\eta]$; then Hölder's inequality implies
\begin{equation*}
\mathbb{E}\{Y(\bs_1)^{\alpha_0+\delta}Y(\bs_2)^{\alpha_0+\delta}\}<\infty,
\end{equation*}
so the bound is integrable. Hence, by dominated convergence,
\begin{equation*}
\frac{\mathbb{P}\{X_F(\bs_1)>x,\ X_F(\bs_2)>x\}}{\bar F_{\epsilon_F}(x)^2}
\rightarrow \mathbb{E}\{Y(\bs_1)^{\alpha_0}Y(\bs_2)^{\alpha_0}\},\qquad \text{as }x\to\infty.
\end{equation*}
Since $\bar F_{\epsilon_F}(x)\sim c_F x^{-\alpha_0}$, we obtain
\begin{equation*}
\mathbb{P}\{X_F(\bs_1)>x,\ X_F(\bs_2)>x\}
\sim c_F^{\,2}\,\mathbb{E}\{Y(\bs_1)^{\alpha_0}Y(\bs_2)^{\alpha_0}\}\,x^{-2\alpha_0},\qquad \text{as }x\to\infty.
\end{equation*}

Repeating the same argument for $X_L(\bs_i)=\epsilon_L(\bs_i)Y(\bs_i)$, $i=1,2$, with
$\bar F_{\epsilon_L}(x)\sim c_L x^{-\alpha_0}$, yields
\begin{equation*}
\mathbb{P}\{X_L(\bs_1)>x,\ X_L(\bs_2)>x\}
\sim c_L^{\,2}\,\mathbb{E}\{Y(\bs_1)^{\alpha_0}Y(\bs_2)^{\alpha_0}\}\,x^{-2\alpha_0}.
\end{equation*}
Therefore,
\begin{equation*}
\frac{\mathbb{P}\{X_F(\bs_1)>x,\ X_F(\bs_2)>x\}}
     {\mathbb{P}\{X_L(\bs_1)>x,\ X_L(\bs_2)>x\}}
\rightarrow
\left(\frac{c_F}{c_L}\right)^{2},
\qquad \text{as }x\to\infty,
\end{equation*}
or equivalently,
\begin{equation*}
\mathbb{P}\{X_F(\bs_1)>x,\ X_F(\bs_2)>x\}
\sim
\left(\frac{c_F}{c_L}\right)^{2}
\mathbb{P}\{X_L(\bs_1)>x,\ X_L(\bs_2)>x\},\qquad \text{as }x\to\infty.
\end{equation*}

Combining the marginal and bivariate results establishes the theorem.
\end{proof}

\section{ELBO derivation}\label{App:ELBO}

Fix a time index $t$ and condition vector $\bc_t$. Recall that the decoder first maps
the latent vector $\bz_t$ to the de–noised process
\[
\by_t(\bc_t) = \bW \bz_t \in \mathbb{R}^{n_s},
\]
and then introduces log–Laplace noise, leading to the conditional CDF
\begin{equation}\label{eq:app_decoder_CDF}
p_{\bphi_d}(\bX_t \le \bx_t \mid \bz_t,\bc_t)
= \prod_{j\in\mathcal{J}_t}
    \left\{
        \tfrac{1}{2}\,x_{jt}^{\alpha_0}\,y_{jt}^{-1}
    \right\}
  \cdot
  \prod_{j\notin\mathcal{J}_t}
    \left\{
        1 - \tfrac{1}{2}\,x_{jt}^{-\alpha_0}\,y_{jt}
    \right\},
\end{equation}
where $y_{jt}$ denotes the $j$th element of $\by_t(\bc_t)$ and
\[
\mathcal{J}_t \;=\; \Bigl\{ j\in\{1,\ldots,n_s\} : 0 < x_{jt}/y_{jt} < 1 \Bigr\}.
\]
Differentiating~\eqref{eq:app_decoder_CDF} with respect to $\bx_t$ gives the conditional
density
\begin{equation}\label{eq:app_decoder_PDF}
    p_{\bphi_d}(\bx_t \mid \bz_t,\bc_t)
    = \prod_{j\in\mathcal{J}_t}
        \frac{\alpha_0 x_{jt}^{\alpha_0 - 1}}{2\,y_{jt}^{\alpha_0}}
      \cdot
      \prod_{j\notin\mathcal{J}_t}
        \frac{\alpha_0 x_{jt}^{-\alpha_0 - 1}}{2\,y_{jt}^{-\alpha_0}}.
\end{equation}
Taking logs and simplifying, this can be written compactly as
\begin{equation}\label{eq:log_like_X_given_ZC}
\begin{split}
    \log p_{\bphi_d}(\bx_t \mid \bz_t,\bc_t)
    &= \sum_{j\in\mathcal{J}_t}
        \Bigl(
            \log\alpha_0
          + (\alpha_0 - 1)\log x_{jt}
          - \log 2
          - \alpha_0 \log y_{jt}
        \Bigr) \\
    &\quad
      + \sum_{j\notin\mathcal{J}_t}
        \Bigl(
            \log\alpha_0
          + (-\alpha_0 - 1)\log x_{jt}
          - \log 2
          + \alpha_0 \log y_{jt}
        \Bigr) \\
    &= \sum_{j=1}^{n_s}
        \Bigl\{
            \log \alpha_0
          - \log 2
          - \log x_{jt}
          - \alpha_0
            \Bigl|
               \log\frac{x_{jt}}{y_{jt}}
            \Bigr|
        \Bigr\},
\end{split}
\end{equation}
where the last equality follows from the identity
\[
\Bigl|
  \log\frac{x_{jt}}{y_{jt}}
\Bigr|
=
\begin{cases}
 -\log(x_{jt}/y_{jt}) = \log(y_{jt}/x_{jt}),
   & j\in \mathcal{J}_t,\\[3pt]
 \phantom{-}\log(x_{jt}/y_{jt}),
   & j\notin \mathcal{J}_t.
\end{cases}
\]

\vspace{0.3cm}
\noindent
\textbf{Prior on the latent process.}
Given $\bc_t$, the prior on the latent vector $\bz_t$ is
\begin{equation}\label{eq:app_prior}
    p_{\bphi_d}(\bz_t \mid \bc_t)
    = \prod_{k=1}^{K}
        h\bigl(z_{kt}; \alpha, \theta_{kt}(\bc_t)\bigr),
\end{equation}
where $h(\cdot;\alpha,\theta_{kt}(\bc_t))$ is the exponentially–tilted
positive–stable density. In particular, we may write
\begin{equation}\label{eq:app_ETPS_density}
    h\bigl(z_{kt};\alpha,\theta_{kt}(\bc_t)\bigr)
    = \frac{1}{2}\,\pi^{-1/2}
      \exp\!\Bigl(
         \theta_{kt}(\bc_t)^{1/2}
      \Bigr)
      z_{kt}^{-3/2}
      \exp\!\left\{
          -\theta_{kt}(\bc_t)\,z_{kt}
          - \frac{1}{4z_{kt}}
      \right\},
\end{equation}
so that
\begin{equation}\label{eq:log_prior_Z}
    \log p_{\bphi_d}(\bz_t \mid \bc_t)
    = \sum_{k=1}^{K}
        \log h\bigl(z_{kt};\alpha,\theta_{kt}(\bc_t)\bigr).
\end{equation}

\vspace{0.3cm}
\noindent
\textbf{Encoder / variational posterior.}
The approximate posterior is specified on the log–scale as
\begin{equation}\label{eq:app_encoder_reparam}
\begin{split}
    \log \bz_t
    &= \log \bmu_t + g(\bc_t)
       + \bsigma_t \odot \boldsymbol{\epsilon}_t,
    \qquad \boldsymbol{\epsilon}_t
        \stackrel{\text{ind}}{\sim} \mathrm{MVN}(\boldsymbol{0},\boldsymbol{I}),\\
    (\bmu_t^\top,\bsigma_t^\top)^\top
    &= \mathrm{EncoderNeuralNet}_{\bphi_e}(\bx_t),
\end{split}
\end{equation}
so that, componentwise,
\[
\log z_{kt} \mid \bx_t,\bc_t \sim
\mathcal{N}\bigl(m_{kt},\sigma_{kt}^2\bigr),
\qquad
m_{kt} := \log \mu_{kt} + g_k(\bc_t).
\]
Hence $q_{\bphi_e}(\bz_t\mid\bx_t,\bc_t)$ is a product of log–normal
densities, and up to an additive constant,
\begin{equation}\label{eq:log_q_Z_given_XC}
\begin{split}
    \log q_{\bphi_e}(\bz_t\mid\bx_t,\bc_t)
    &= \sum_{k=1}^{K}
       \log \Bigl\{
          \mathrm{Lognormal}\bigl(
             z_{kt}; m_{kt}, \sigma_{kt}^2
          \bigr)
       \Bigr\} \\
    &= -\sum_{k=1}^{K}
        \Biggl[
            \log z_{kt}
          + \log \sigma_{kt}
          + \frac{
              \bigl(\log z_{kt} - m_{kt}\bigr)^2
            }{2\sigma_{kt}^2}
        \Biggr]
        \;+\; \text{const}.
\end{split}
\end{equation}
Using the reparameterization
\[
\log z_{kt} = m_{kt} + \sigma_{kt}\,\epsilon_{kt},
\qquad
\epsilon_{kt}\sim N(0,1),
\]
one can equivalently express $-\log q_{\bphi_e}$ in terms of
$(\bsigma_t,\boldsymbol{\epsilon}_t)$ as
\begin{equation}\label{eq:minus_log_q_eps_form}
    -\log q_{\bphi_e}(\bz_t\mid\bx_t,\bc_t)
    = \sum_{k=1}^{K}
        \Bigl(
            \log \sigma_{kt}
          + \tfrac{1}{2}\,\epsilon_{kt}^2
        \Bigr)
      \;+\; \text{const},
\end{equation}
which is the form used in Monte Carlo estimation of the ELBO.

\vspace{0.3cm}
\noindent
\textbf{Per–time–step conditional ELBO.}
For fixed $t$ and condition $\bc_t$, the ELBO is
\begin{equation}\label{eq:app_ELBO_def}
\begin{split}
    \mathcal{L}_{\bphi_e,\bphi_d}(\bx_t \mid \bc_t)
    &= \mathbb{E}_{q_{\bphi_e}(\bz_t\mid\bx_t,\bc_t)}
       \Bigl[
           \log p_{\bphi_d}(\bx_t \mid \bz_t,\bc_t)
         + \log p_{\bphi_d}(\bz_t \mid \bc_t)
         - \log q_{\bphi_e}(\bz_t\mid\bx_t,\bc_t)
       \Bigr] \\
    &= \mathbb{E}_{q_{\bphi_e}(\bz_t\mid\bx_t,\bc_t)}
       \Biggl[
           \sum_{j=1}^{n_s}
             \Bigl\{
                \log \alpha_0
              - \log 2
              - \log x_{jt}
              - \alpha_0\Bigl|
                  \log\frac{x_{jt}}{y_{jt}}
                \Bigr|
             \Bigr\} \\
    &\qquad\qquad\qquad\quad
         + \sum_{k=1}^{K}
             \log h\bigl(z_{kt};\alpha,\theta_{kt}(\bc_t)\bigr)
         - \log q_{\bphi_e}(\bz_t\mid\bx_t,\bc_t)
       \Biggr],
\end{split}
\end{equation}
where $y_{jt}$ is the $j$th component of $\by_t(\bc_t) = \boldsymbol{W}\bz_t$.
In practice, the expectation in~\eqref{eq:app_ELBO_def} is approximated
via Monte Carlo using the reparameterization~\eqref{eq:app_encoder_reparam},
and the temporal smoothness penalty in \eqref{eqn:cXVAE_ELBO_1} is then
subtracted to obtain the final objective
$\mathcal{L}^{\star}_{\bphi_e,\bphi_d}(\bx_t\mid\bc_t)$ used in training.

\setcounter{figure}{0} 
\renewcommand{\thefigure}{D.\arabic{figure}} 


\section{FWI Dataset}\label{App:FWI}
The FWI dataset provides daily Fire Weather Index values for inland regions globally. Because our goal is not to develop or evaluate imputation strategies, we restrict our analysis to a spatial domain with complete observational coverage. This ensures that missing data do not influence the modeling pipeline.

\paragraph{Remove Seasonality:} To clean and de-seasonalize the dataset, we follow a procedure similar to that described in the Appendix of \cite{zhang2023flexible}. Let 
\begin{align*}
    \bX(\bs_j) = (X_1(\bs_j),\ldots,X_{N}(\bs_j))^\top
\end{align*}
denote the daily FWI observations at location $\bs_j$, where $N=3,867$ corresponds to all days from May 1, 2014 through November 30, 2024, and $j=1,\ldots,n_s$.

Seasonality is removed separately at each location. For a given site $\bs_j$, we regress $\bX(\bs_j)$ on a set of cubic spline basis function in time. To ensure the spatial continuity, the regression is not fit using the data from whole region; instead, we pool information from a local neighborhood. Specifically, for each location $\bs_j$ we define the neighboring set
\begin{align*}
    \mathcal{S}_j = \{ \bs_i:||\bs_i - \bs_j||<r, i=1,\ldots,n_s \}
\end{align*}
where $r=60$ km, note the $j$th location is also in the neighboring set. Let us denote the number of neighbors in $\mathcal{S}_j$ is $N_j$. Therefore, $\bX(\mathcal{S}_j)$ of dimension $N_j \cdot N \times1 $ is the response for the regression.

Second, we construct the matrix $\boldsymbol{M}=(\mathbf{1}_{N}, \boldsymbol{t}, \mathbf{B})$, where $\mathbf{1}_N$ is the column vector of 1s of length $N$ for the intercept term, $\boldsymbol{t}=(1,\ldots,N)^\top$ is used to fit linear time trend. The columns of $\mathbf{B}$ are 12 cyclic cubic spline basis functions, one representing each month of the year, designed to mimic the smooth monthly cycle observed in the FWI series. These splines are defined over the day index (1-365) with knots placed at evenly spaced quantiles, producing 12 smooth curves that cover everyday of the year (see Figure~\ref{appfig:spline}). This cyclic component also ensures the continuity and smoothness at the boundaries between day 365 and day 1. Then, we vertically stack the matrix $\boldsymbol{M}$ for $N_j$ times to build the design matrix $\boldsymbol{M}_j$. After regressing $\bX(\mathcal{S}_j)$ on $\boldsymbol{M}_j$, we have the fitted values $\hat{\bX}(\mathcal{S}_j)$.

\setcounter{figure}{0} 
\renewcommand{\thefigure}{F.\arabic{figure}} 

\begin{figure}
    \centering
    \includegraphics[width=0.8\linewidth]{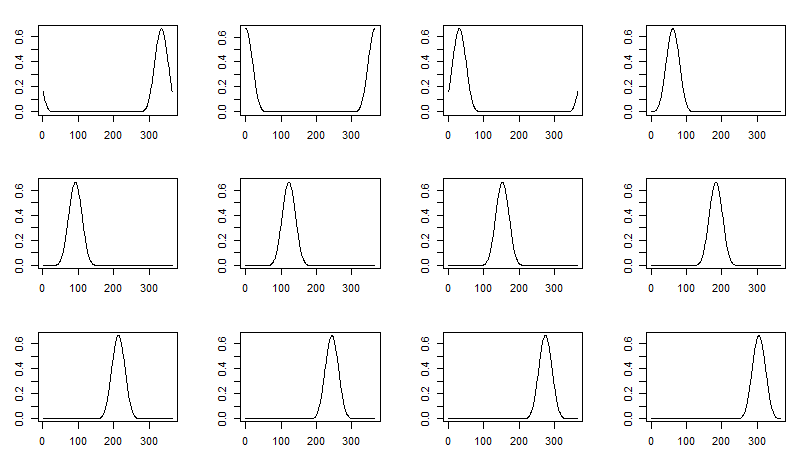}
    \caption{Specially designed 12 cyclic cubic spline basis functions is used for capture repeated seasonal patterns.}
    \label{appfig:spline}
\end{figure}

To model the residuals $\bX(\mathcal{S}_j) - \hat{\bX}(\mathcal{S}_j)$, we use the first two columns of $\boldsymbol{M}_j$ (denote as $\boldsymbol{M}^{\sigma}_j$). With assuming it follows multivariate normal distribution with zero mean vector and variance matrix $\operatorname{diag}(\boldsymbol{\epsilon}_j^2)$. Specifically,
\begin{align*}
    \bX(\mathcal{S}_j) - \hat{\bX}(\mathcal{S}_j) &\sim N(\mathbf{0}, \operatorname{diag}(\boldsymbol{\epsilon}_j^2)), \\
    \log \boldsymbol{\epsilon}_j &= \boldsymbol{M}^{\sigma}_j \times (\beta_1, \beta_2)^\top.
\end{align*}

Then we can estimate parameters $(\beta_1, \beta_2)^\top$ via optimizing the multivariate normal density function:
\begin{align*}
    (\hat{\beta}_1, \hat{\beta}_2)^\top = \argmin_{(\beta_1, \beta_2)^\top} \left\{ -\frac{1}{2} \log \mathbf{1}^\top \boldsymbol{\epsilon}_j^2 -\frac{1}{2} (\bX(\mathcal{S}_j) - \hat{\bX}(\mathcal{S}_j))^\top \operatorname{diag}(\boldsymbol{\epsilon}_j^{-2}) (\bX(\mathcal{S}_j) - \hat{\bX}(\mathcal{S}_j)) \right\}.
\end{align*}

With the estimated parameters $(\hat{\beta}_1, \hat{\beta}_2)^\top$ in hand, the corresponding estimated standard deviations are given by
\begin{align*}
    \hat{\boldsymbol{\epsilon}}_j=\exp\{ \boldsymbol{M}^{\sigma}_j \times (\hat{\beta}_1, \hat{\beta}_2)^\top \}.
\end{align*}
The vector $\hat{\boldsymbol{\epsilon}}_j$ contains the estimated standard deviations for all neighboring sites of location $\bs_j$, stacked vertically, and has dimension $N_j \cdot N \times 1$. From this vector, we extract the entry corresponding to site $\bs_j$ itself, denoted $\hat{\boldsymbol{e}}_j$. Finally, the daily records at location $\bs_j$ is de-trended by standardizing the residuals as
\begin{align}\label{eqn:detrend}
    \bX^*(\bs_j) = \frac{\bX(\bs_j) - \hat{\bX}(\bs_j)}{\hat{\boldsymbol{e}}_j}.
    \tag{F.1}
\end{align}
This procedure is repeated for all locations in the target region.

\paragraph{Marginal distributions of the monthly maxima:} After removing seasonality using the normalization in~\ref{eqn:detrend}, we extract the monthly maxima from $\bX^*(\bs_j)$ at site $\bs_j$ and denote them as $\bm_j=(m_{j1},\ldots,m_{jn_t})$, where $n_t=127$ is the number of months from May 1, 2014 through November 30, 2024. Before applying our model, we require an appropriate marginal distribution for these monthly maxima so that they can be transformed to a Pareto-type scale. We consider two candidates: the generalized extreme value (GEV) distribution and the general non-central $t$ distribution. To compare them, we employ a $\chi^2$ goodness-of-fit tests, which offers flexibility in specifying both the number of bins and the degrees of freedom.

The $\chi^2$ goodness-of-fit test at site $\bs_j$ proceeds as follows:
\begin{enumerate}
    \item \textbf{Define bins:} Construct $n_I+1$ equally spaced cut points spanning the range of the monthly maxima at $\bs_j$, and then get $n_I$ intervals.
    \item \textbf{Observed frequencies:} Count the number of maxima falling into each interval, denoted $O_i$, for $i=1,\ldots,n_I$.
    \item \textbf{Fit candidate models:} Fit both GEV and $t$ distributions to the monthly maxima and obtain parameter estimates.
    \item \textbf{Expected frequencies:} For each model, compute the expected frequency
    \begin{align*}
        E_i = n_t p_i,
    \end{align*}
    where $p_i$ is the probability increment of the fitted distribution in each interval.
\end{enumerate}

Viewing the monthly maxima as a multinomial sample with $n_t$ trials and $n_I$ categories, the generalized likelihood-ratio statistic for testing
\begin{align*}
    H_0:(p_1, \ldots, p_{n_I})^\top \text{are the true event probabilities}
\end{align*}
is given by
\begin{align*}
    \sum_{i=1}^{n_I} O_i \log(O_i/E_i) \xrightarrow{d} \chi_{d}^2 \quad \text{as} \quad n_t \rightarrow \infty,
\end{align*}
where $d = n_I-4$ for the GEV model (three parameters: location, scale and shape) and $d = n_I-3$ for the $t$ model (two parameters: non-centrality parameter and degrees of freedom). 

\begin{figure}
    \centering
    \includegraphics[width=0.85\linewidth]{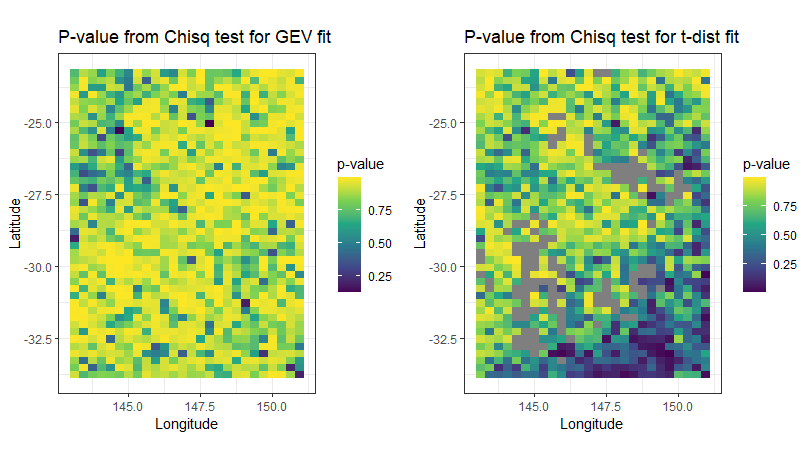}
    \caption{Heatmaps of $p$-values from $\chi^2$ goodness-of-fit tests. Left for GEV model and right for $t$ model.}
    \label{fig:app_pval}
\end{figure}

Applying this procedure at all locations yields the $p$-values heatmap in Figure~\ref{fig:app_pval}. For the GEV model (left), all locations produce $p$-values greater than 0.05, indicating an excellent fit across the study region. For the non-central $t$ model (right panel), missing $p$-values correspond to failures in parameter estimation, which prevent the computation of the $\chi^2$ statistic. Additionally, several sites in northern New South Wales fail the goodness-of-fit tests, with $p$-values below 0.05. 

Overall, the diagnostics from the $\chi^2$ goodness-of-fit test demonstrate that the GEV distribution provides a more reliable and robust marginal model for the monthly maxima in this domain.

\paragraph{Marginal transformation:} Once the goodness-of-fit tests are completed, we obtain site-specific GEV parameter estimates $\hat{\mu}_j, \hat{\sigma}_j, \hat{\xi}_j$ for all locations. Before applying our model to the monthly maxima, a monotonic transformation is required to map the data to a Pareto-type scale.

For each site $j=1,\ldots, n_s$, define the upper bound of the GEV distribution as
\begin{align*}
    \beta_{j} = \mu_{j} - \sigma_{j}/\xi_{j},
\end{align*}
The transformed value of the monthly maxima $m_{jt}$ is then given by
\begin{align*}
    x_{jt} = \left\{ \frac{(m_{jt} - \beta_j)\cdot \xi_j}{\sigma_j} \right\} ^{1/\xi_j},
\end{align*}
if $\xi_j > 0$ and 
\begin{align*}
    x_{jt} = \left\{ \frac{\sigma_j}{(\beta_j - m_{jt})\cdot |\xi_j|} \right\} ^{1/|\xi_j|},
\end{align*}
if $\xi_j < 0$. Collecting the transformed values across space, we obtain
\begin{align*}
    \bx_t = (x_{1t}, x_{2t}, \ldots,x_{n_st})^\top, \quad t=1,\ldots,n_t,
\end{align*}
which serve as the model input for the cXVAE.

\end{document}